\documentclass[11pt,letterpaper]{article}
\usepackage[T1]{fontenc}
\usepackage[utf8]{inputenc}
\usepackage[margin=1in]{geometry}
\usepackage{times}
\usepackage{microtype}
\usepackage[authoryear,round]{natbib}
\setcitestyle{citesep={;},aysep={,},yysep={;}}
\usepackage{xcolor}
\usepackage{xurl}
\usepackage{graphicx}

\usepackage{amsmath,amsfonts,bm}

\def\eqref#1{equation~\ref{#1}}

\def\1{\bm{1}}

\DeclareMathAlphabet{\mathsfit}{\encodingdefault}{\sfdefault}{m}{sl}
\SetMathAlphabet{\mathsfit}{bold}{\encodingdefault}{\sfdefault}{bx}{n}

\newcommand{\KL}{D_{\mathrm{KL}}}

\usepackage{amsmath,amsthm}
\newtheorem{theorem}{Theorem}

\newtheorem{proposition}{Proposition}
\newtheorem{definition}{Definition}
\newtheorem{remark}{Remark}
\newtheorem{assumption}{Assumption}
\newtheorem{lemma}{Lemma}
\usepackage{booktabs}
\usepackage{multirow}
\usepackage{float}
\newcommand{\enc}{\operatorname{enc}}

\newcommand{\TC}{\operatorname{TC}}
\newcommand{\Rb}{\mathbb{R}}

\newcommand{\Ncal}{\mathcal{N}}
\newcommand{\Lcal}{\mathcal{L}}
\newcommand{\zvec}{\bm{z}}
\newcommand{\ovec}{\bm{o}}
\newcommand{\avec}{\bm{a}}
\newcommand{\uvec}{\bm{u}}

\newcommand{\SIGReg}{\operatorname{SIGReg}}

\usepackage[hidelinks]{hyperref}
\hypersetup{
  pdftitle={LRC-JEPA: Disentangling Dynamics and Residual Context for Efficient World Models},
  pdfauthor={Luzhe Huang, Lei Chu, Jingyi Liang, Yuhuan Zhao}
}

\title{LRC-JEPA: Disentangling Dynamics and Residual Context for Efficient World Models}

\author{%
  \normalsize Luzhe Huang\textsuperscript{1}\quad
  \normalsize Lei Chu\textsuperscript{2}\quad
  \normalsize Jingyi Liang\textsuperscript{1}\quad
  \normalsize Yuhuan Zhao\textsuperscript{1}\\[0.5em]
  \small \textsuperscript{1}Independent researcher\qquad
  \small \textsuperscript{2}USC
}
\date{}

\begin{document}

\maketitle

\begin{abstract}
 Compact JEPA world models enable efficient latent-space planning, but
low-dimensional representation trained under reward-free self-supervision must encode both action-conditioned dynamics and
predictable visual context. This competition can entangle controllable state
with high-rank nuisance appearance and degrade planning as scenes become more complex. We
introduce LRC-JEPA, a lightweight end-to-end world model that routes information
into a compact predictive latent $\mathbf{z}$ and learned-query residual-context
embeddings $\mathbf{u}$. Only $\mathbf{z}$ is propagated by the dynamics model
and used for planning, while $\mathbf{u}$ captures temporally persistent
information for cross-attention reconstruction; a differentiable residual
connection encourages the latent to retain complementary dynamic content.
Under explicit assumptions, we show that the resulting representation is
sufficient, minimal, nuisance-invariant, and disentangled. Across four simulated
control environments, LRC-JEPA improves average planning success over a
parameter-matched JEPA baseline by 9 percentage points and matches or exceeds
substantially larger pretrained models. On the real-world
Bridge-v2 set, its 5.5M-parameter active encoder outperforms DINO-WM (22.1M) and V-JEPA2 (303.9M) encoders
while also enabling faster planning. Physical-state probes, reconstruction
interventions, and ablations confirm the effectiveness of LRC-JEPA's representation disentanglement.
\end{abstract}

\section{Introduction}
\label{sec:introduction}

Learning an internal model of the physical world is a central objective of embodied intelligence. World models infer state from high-dimensional observations, predict the effects of candidate actions, and support planning toward unseen goals. Learning them from offline, reward-free interaction is attractive because the same trajectories can serve multiple tasks without reward annotation, online exploration, or policy retraining \cite{ha2018worldmodels,sobal2025pldm}. Visual world models have consequently shifted from pixel prediction \cite{finn2016unsupervised,ebert2018visualforesight} toward dynamics in compressed representation spaces \cite{hafner2019planet,hafner2023dreamerv3,hansen2024tdmpc2}, enabling policy learning or efficient test-time optimization.
One line of latent world models predicts rich patch-level features from large pretrained encoders, as in DINO-WM and V-JEPA2 \cite{zhou2024dinowm,assran2025vjepa2}. These features retain fine-grained semantic and spatial information but make repeated planning expensive \cite{maes2026leworldmodel}. Compact alternatives jointly learn dynamics over one or a few scene embeddings \cite{hafner2019planet,sobal2025pldm,maes2026leworldmodel}; LeWorldModel (LWM), for example, predicts a single latent state and avoids pixel reconstruction. However, this efficiency creates a \emph{capacity--compactness trade-off}: low-dimensional encoders can discard action-relevant transition information or entangle controllable factors with predictable but weakly action-dependent context \cite{pan2022isodream,cui2026generalization}. The problem is acute in reward-free image-goal planning, where future goals may depend on context that cannot be discarded a priori. Figure~\ref{fig:capacity_complexity_tradeoff} exposes this failure mode: a single-latent JEPA degrades under PCA compression and task-irrelevant visual perturbations, indicating entanglement of nuisance with control-relevant dynamics and inefficient use of its limited capacity.

Motivated by optimal representation theory \cite{achille2018emergence,achille2018information}, we propose \textbf{LRC-JEPA}, a lightweight end-to-end JEPA world model with two functionally distinct representations: a compact predictive latent $\mathbf{z}_t$ and learned-query residual-context embeddings $\mathbf{u}_t$. Our contributions are:
\begin{itemize}
    \item We introduce a functional routing architecture that assigns predictive dynamics to compact $\mathbf{z}$ and slow-varying or persistent context to complementary $\mathbf{u}$.
    \item Under explicit assumptions, we show that LRC-JEPA satisfies sufficiency, minimality, nuisance invariance, and coordinate disentanglement, with $\mathbf{u}_t$ and $\mathbf{z}_t$ sufficient for context and dynamic state, respectively.
    \item Across simulated and real-world settings, LRC-JEPA's 5.5M active parameter encoder improves over a parameter-matched JEPA and surpasses up to 55$\times$ larger pretrained competitors on key planning metrics and planning time consumption. Probes and interventions confirm the intended allocation of kinematics to $\mathbf{z}$ and residual context to $\mathbf{u}$.
\end{itemize}

\section{Related Work}
\label{sec:related_work}

\paragraph{Visual world models and model-based control.}
World models predict action consequences for planning and policy learning \cite{ha2018worldmodels}. Early visual methods predict future pixels for robotic control \cite{finn2016unsupervised,ebert2018visualforesight}, while recent foundation models scale video generation to broad physical-AI data \cite{nvidia2025cosmos}. Such models preserve visual detail but must also spend substantial capacity on appearance variation that may be irrelevant to a particular control objective. Latent approaches instead support planning or policy learning through compressed states \cite{hafner2019planet,hafner2023dreamerv3,hansen2022tdmpc,hansen2024tdmpc2}. PlaNet and Dreamer combine latent dynamics with observation and reward prediction, whereas TD-MPC methods emphasize task-oriented latent dynamics, reward prediction, and value estimation. These objectives provide signals for retaining task-relevant information. Reward-free visual planning poses a different challenge: the representation must preserve action-dependent state and distinguish image goals without reward supervision. LRC-JEPA addresses this setting by learning a compact predictive state alongside a separate context representation used for reconstruction during training.

\paragraph{Joint-embedding predictive architecture (JEPA).}
JEPAs learn by predicting target representations rather than reconstructing every input detail. I-JEPA introduces masked representation prediction for images \cite{assran2023ijepa}, and V-JEPA extends this approach to video \cite{bardes2024vjepa}. Existing methods prevent collapse with asymmetric or stop-gradient target encoders. LeJEPA instead uses teacher-free distributional regularization \cite{balestriero2025lejepa}, which LWM applies to compact action-conditioned dynamics learned from pixels \cite{maes2026leworldmodel}. Alternatively, DINO-WM and V-JEPA2 retain much richer pretrained features \cite{oquab2024dinov2,zhou2024dinowm,assran2025vjepa2}. These methods successfully establish feature prediction for control. However, preventing collapse does not inherently determine which physical variables a predictive representation preserves.

\paragraph{Functional separation in representation learning.}
Separating appearance and dynamics has a long history in video generation and prediction \cite{denton2017disentangled,tulyakov2018mocogan,hsieh2018ddpae}; MotionRNN further distinguishes motion trends from transient variation \cite{wu2021motionrnn}. MC-JEPA jointly learns motion and content features through optical-flow and self-supervised representation objectives \cite{bardes2023mcjepa}. For control, bisimulation uses reward and transition equivalence to suppress irrelevant variation \cite{zhang2021bisimulation}. Without rewards, predictive objectives can favor slowly varying distractors over changing task state \cite{sobal2022slowfeatures}. Recent work addresses related allocation problems through progression--content subspaces \cite{thil2026subspacejepa} or Gaussian regularization of temporally centered residuals \cite{liu2026temporallycentered}. LRC-JEPA uses a different division of responsibility: action-conditioned prediction shapes $\mathbf{z}$, while reconstruction and temporal invariance encourage $\mathbf{u}$ to retain complementary, slowly varying context. This is a functional bias rather than a prescribed semantic partition; slowly varying quantities can still matter for control. The context stream and decoder support representation learning, while planning operates only on $\mathbf{z}$, connecting the separation directly to inference efficiency.

\section{Method}
\label{sec:method}

We first use optimal representation theory to identify a limitation of single-latent JEPAs, then present LRC-JEPA and establish its minimality and sufficiency properties.

\subsection{Theoretical Analysis Foundations}

\begin{figure*}[t]
    \centering
    \includegraphics[width=\textwidth]
        {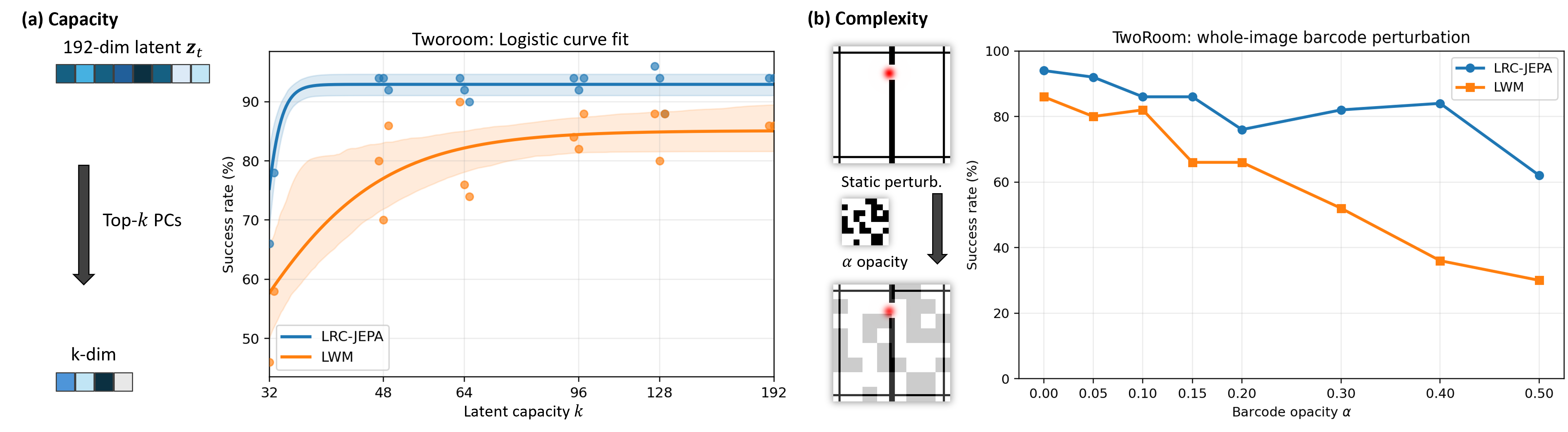}
    \caption{
        \textbf{Zero-shot capacity--complexity stress tests for compact
        JEPA planning.}
        \emph{Left:} Starting from the original $192$-dimensional
        predictive latent, PCA is fitted on clean representations and
        latents are projected onto top-$k$ components during planning.
        \emph{Right:} In TwoRoom, a trajectory-persistent grayscale
        pattern is blended into both current and goal observations with
        opacity $\alpha$, where $\alpha=0$ denotes the clean environment. Details are provided in Appendix \ref{sec:trade-off}.
    }
    \label{fig:capacity_complexity_tradeoff}
\end{figure*}

We adopt the Achille--Soatto criteria of sufficiency, minimality, invariance, and disentanglement \cite{achille2018information}; definitions appear in Appendix~\ref{app:defs}.

\begin{theorem}[Achille--Soatto criterion;
{\citealp[Prop.~3.1 and \S5]{achille2018emergence}}]
\label{thm:criterion}
Let $Y$ be a task, $N$ a nuisance ($I(Y;N)=0$), and $Z$ a sufficient
representation of the input $X$ such that the mutual information $I(Y;Z)=I(Y;X)$, with $N\to X\to Z$.
(a)~\emph{Sufficiency, minimality $\Rightarrow$ Invariance:} every nuisance satisfies
$I(Z;N)\le I(Z;X)-I(X;Y)$ (Remark~\ref{rem:residual}); since $I(X;Y)$ is constant on given data, a
sufficient $Z$ is maximally invariant w.r.t $N$ if $I(Z;X)$ minimizes.
(b)~\emph{Minimality $\Rightarrow$ coordinate disentanglement:} For a stochastic-network
model $Z=g_z(X)$, the sum of mutual information and total correlation (TC) is upper bounded by the information of training data, so a weight-information regularization enforcing minimality drives the total correlation $\TC(Z)$ down with it.
Thus, invariance and coordinate disentanglement follow one mechanism: a penalty that
decreases $I(Z;X)$ while sufficiency is maintained.
\end{theorem}

We next characterize the single-latent LWM under these criteria.

\begin{proposition}
\label{prop:LeWorldModel}
\textbf{Sufficiency.} Under Assumptions~\ref{ass:surrogate}, \ref{ass:1gauss}, \ref{ass:sigreg}, the population SIGReg objective vanishes if and only if $q(\zvec_t)=\Ncal(0,I_d)$, in which case $\TC(Z_t)=0$ and the coordinates are mutually independent and disentangled. Moreover, the latent representation learned by LWM is sufficient for predicting the underlying state. \textbf{No guaranteed minimality.} For every $\lambda\ge0$, $\SIGReg$ is unchanged by encoder modifications that
preserve $q(\zvec_t)$, including those changing the mutual information $I(Z_t;X_t)$. In fact, $I(Z_t;X_t)$ is
pushed up when $I(Z_t;N)$ maximizes; and
$\Lcal_{\mathrm{pred}}(Z_t)$ only rewards information about the transition but not
penalizes $I(Z;N)$. Hence, during training no term of $\Lcal_{\mathrm{LWM}}$ decreases $I(Z_t;X_t)$,
and the minimizers of $\Lcal_{\mathrm{LWM}}$ does not generally guarantee minimal prediction-sufficient
statistics.
\end{proposition}
The proof is in Appendix~\ref{sec:anal_lwm}. SIGReg's unique population minimizer $\Ncal(0,I_d)$ prevents collapse, fixes scale, and decorrelates features, but does not prevent control-irrelevant nuisance information from occupying the latent. This limitation is consistent with Figure~\ref{fig:capacity_complexity_tradeoff}.

\subsection{Model Architecture}
\label{sec:lrc_architecture}

\begin{figure*}[t]
    \centering
    \includegraphics[width=\textwidth]
        {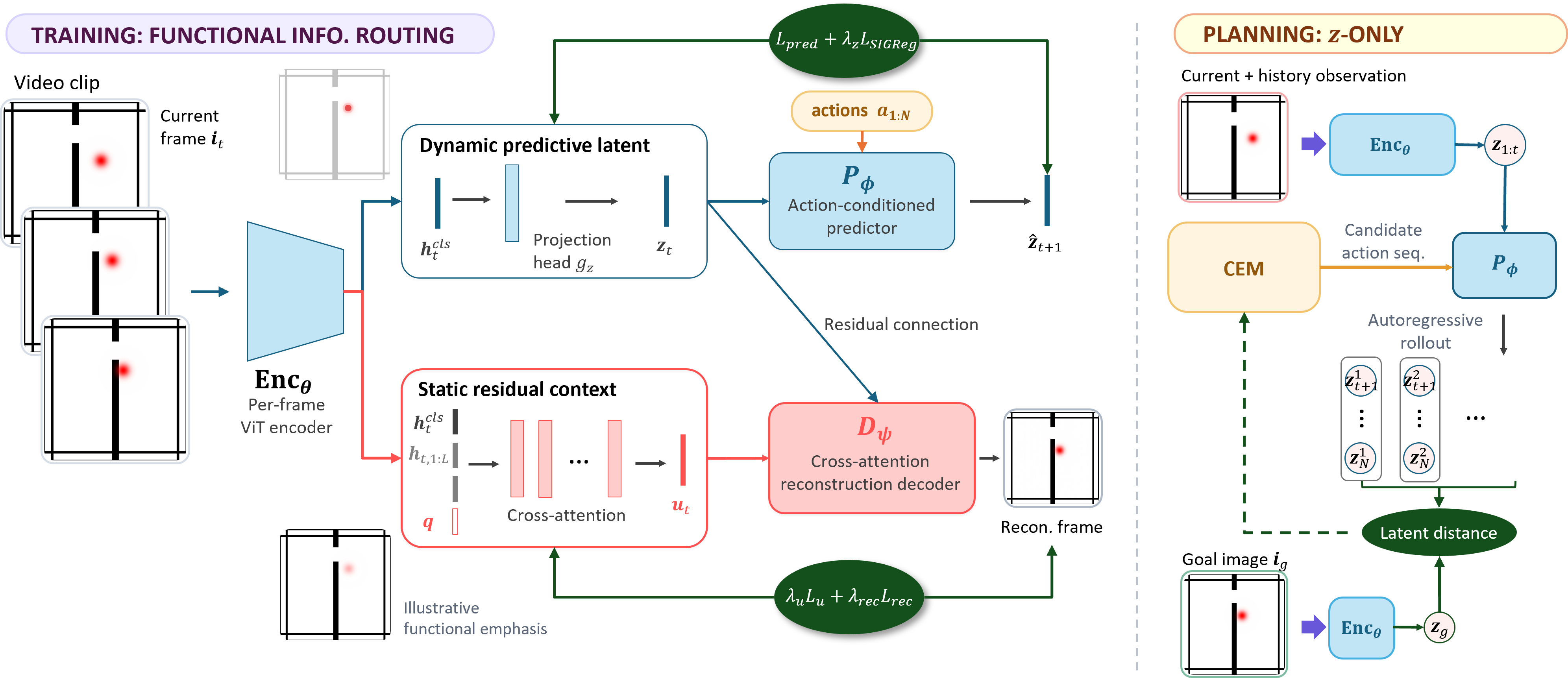}
    \caption{\textbf{LRC-JEPA architecture.} During training, functional information routing separates a compact, dynamics-predictive latent $z_t$ from an auxiliary representation $u_t$ that supports reconstruction through cross-attention. This separation is the key novelty of our method: it concentrates control-relevant dynamics in $z$, enabling efficient $z$-only planning with CEM and latent-distance objectives at test time.}
    \label{fig:LRC-JEPA}
\end{figure*}

An observation $\mathbf{i}_t\in\mathbb{R}^{3\times H\times W}$ contains action-dependent state and slowly varying context such as background, lighting, camera configuration, and appearance. LRC-JEPA assigns them complementary roles: $\mathbf{z}_t$ is a compact \emph{predictive latent state} used by the prediction model and planning objective, whereas $\mathbf{u}_t$ is a bottlenecked \emph{residual-context embedding} used only for reconstruction.

\paragraph{Shared visual encoder.}

Each frame ($\mathbf{i}_t$) is processed independently by a standard Vision Transformer \cite{dosovitskiy2020image}, denoted by 
$\operatorname{Enc}_{\theta}$. We represent its final-layer class and patch tokens as
\begin{equation}
    \left[
        \mathbf{h}^{\mathrm{cls}}_t,
        \mathbf{h}_{t,1},\ldots,\mathbf{h}_{t,L}
    \right]
    =
    \operatorname{Enc}_{\theta}(\mathbf{i}_t)
    \in \mathbb{R}^{(L+1)\times d_h},
    \label{visual_encoder}
\end{equation}
where $\mathbf{h}^{\mathrm{cls}}_t,\mathbf{h}_{t,\ell}$ denote the class token and the $\ell$-th patch token, respectively. 
The predictive latent is obtained from the class token using a
two-layer projection head,
\begin{equation}
    \mathbf{z}_t
    =
    g_z\!\left(\mathbf{h}^{\mathrm{cls}}_t\right)
    \in \mathbb{R}^{d_z},
    \qquad
    g_z
    =
    \operatorname{Linear}
    \circ \operatorname{GELU}
    \circ \operatorname{BN}
    \circ \operatorname{Linear},
    \label{eq:z_encoder}
\end{equation}

\paragraph{Learned-query residual context.}

To retain spatial and appearance information without passing all patch tokens to the dynamics model, we attentively pool the complete token set using $Q$ learnable queries
$\mathbf{q}\in\mathbb{R}^{Q\times d_u}$. With
$\mathbf{H}_t=[\mathbf{h}^{\mathrm{cls}}_t;
\mathbf{h}_{t,1:L}]$, $W_u$ as a linear projection from ViT feature to context token domain, the context encoder is
\begin{align}
    \widetilde{\mathbf{u}}_t
    &=
    \operatorname{MHA}
    \left(
        \operatorname{LN}(\mathbf{q}),
        \operatorname{LN}(W_u\mathbf{H}_t),
        \operatorname{LN}(W_u\mathbf{H}_t)
    \right), \\
    \mathbf{u}_t
    &=
    \widetilde{\mathbf{u}}_t
    +
    \operatorname{MLP}(\widetilde{\mathbf{u}}_t)
    \in \mathbb{R}^{Q\times d_u}.
    \label{eq:u_encoder}
\end{align}
Thus, $\mathbf{u}_t$ summarizes the spatial feature map through a constrained reconstruction-only bottleneck, discouraging frame-wise shortcuts.

\paragraph{Action-conditioned latent predictor.}

An action encoder $g_a$ maps the normalized action associated with each
visual transition to an embedding $\mathbf{e}_t$,
\begin{equation}
    \mathbf{e}_t = g_a(\mathbf{a}_t)\in\mathbb{R}^{d_z}.
\end{equation}
When one visual transition spans $F$ low-level control steps, the $F$
actions are concatenated before being passed to $g_a$. For
multi-dataset training, dataset-specific action adapters map
heterogeneous action spaces into the common embedding space \cite{maes2026leworldmodel}.

Given a history of $T$ predictive states, a causal transformer $P_{\phi}$ predicts the subsequent states by teacher forcing:
\begin{align}
    \widetilde{\mathbf{z}}_{1:T}
    &=
    P_{\phi}
    \left(
        \mathbf{z}_{1:T}
        +
        \mathbf{r}_{1:T};
        \mathbf{e}_{1:T}
    \right),\\
    \widehat{\mathbf{z}}_{t+1}
    &=
    g_p(\widetilde{\mathbf{z}}_t),
    \qquad t=1,\ldots,T.
    \label{eq:latent_predictor}
\end{align}
Here $\mathbf{r}_{1:T}$ are learned temporal embeddings and $g_p$ has the same structure as $g_z$. Causal self-attention restricts position $t$ to $\mathbf{z}_{1:t}$, while AdaLN-Zero is adapted to every transformer block to condition on actions \cite{peebles2023scalable}. Crucially, $\mathbf{u}_t$ never enters $P_{\phi}$: only $\mathbf{z}_t$ is rolled forward and used for planning.

\paragraph{Cross-attention reconstruction decoder.}

Reconstruction is used to shape the allocation of information between
the two latent streams. For each frame, the decoder context is
\begin{equation}
    \mathbf{c}_t
    =
    \left[
        \mathbf{z}_t;
        \mathbf{u}_{t,1};
        \ldots;
        \mathbf{u}_{t,Q}
    \right].
    \label{eq:decoder_context}
\end{equation}
We use one learned query $\mathbf{s}^{(0)}_n$ for each output image patch. The $l$-th decoder layer applies
\begin{align}
    \mathbf{s}^{(\ell)}_t
    &\leftarrow
    \mathbf{s}^{(\ell-1)}_t
    +
    \operatorname{MHA}
    \left(
        \operatorname{LN}(\mathbf{s}^{(\ell-1)}_t),
        \operatorname{LN}(W_c\mathbf{c}_t),
        \operatorname{LN}(W_c\mathbf{c}_t)
    \right),\\
    \mathbf{s}^{(\ell)}_t
    &\leftarrow
    \mathbf{s}^{(\ell)}_t
    +
    \operatorname{MLP}
    \left(\operatorname{LN}(\mathbf{s}^{(\ell)}_t)\right).
\end{align}
A linear output head maps each final query to a flattened RGB patch,
which is unpatchified to $\widehat{\mathbf{i}}_t$.

\subsection{Training Objective}
\label{sec:lrc_objective}

LRC-JEPA is trained with four complementary objectives: one-step latent prediction, anti-collapse regularization, temporal context regularization, and patch reconstruction. The action-conditioned predictor maps $\mathbf{z}_{1:T}$ to targets $\mathbf{z}_{2:T+1}$ using
\begin{equation}
    \mathcal{L}_{\mathrm{pred}}
    =
    \frac{1}{BTd_z}
    \sum_{b=1}^{B}
    \sum_{t=1}^{T}
    \left\|
        \widehat{\mathbf{z}}^{(b)}_{t+1}
        -
        \mathbf{z}^{(b)}_{t+1}
    \right\|_2^2.
    \label{eq:pred_loss}
\end{equation}
Following LeJEPA and LWM \cite{balestriero2025lejepa, maes2026leworldmodel}, we regularize the predictive latent with SIGReg loss $\mathcal{L}_{\mathrm{SIG}}$, without a stop-gradient target or momentum encoder. We further apply a temporal VICReg-style objective \cite{bardes2022vicreg} $\mathcal{L}_u$ to encourage the residual context $\mathbf{u}_t$ to remain stable within a trajectory while retaining variation across trajectories. Finally, a patch reconstruction loss
\begin{equation}
\mathcal{L}_{\mathrm{rec}}
=
\operatorname{MSE}
\left(
D_{\psi}(\mathbf{z}_t,\mathbf{u}_t),
\mathbf{i}_t
\right)
\label{eq:rec_loss}
\end{equation}
encourages both latent streams to preserve scene information. The complete objective is
\begin{equation}
\mathcal{L}_{\mathrm{LRC\text{-}JEPA}}
=
\mathcal{L}_{\mathrm{pred}}
+
\lambda_z\mathcal{L}_{\mathrm{SIG}}
+
\lambda_u\mathcal{L}_{u}
+
\lambda_{\mathrm{rec}}\mathcal{L}_{\mathrm{rec}}.
\label{eq:lrc_total}
\end{equation}
We use $(\lambda_z,\lambda_u,\lambda_{\mathrm{rec}})=(0.05,0.05,0.2)$. Full definitions and implementation details of each regularizer are provided in Appendix~\ref{sec:imp_details}.

\subsection{Theoretical Analysis for The Proposed Method}

LRC-JEPA splits the representation into a predictive latent
$\zvec_t\in\Rb^{d_z}$ (SIGReg-regularized, rolled forward by the
action-conditioned predictor, alone used for planning) and a
query-bottlenecked \emph{residual context} $\uvec_t$ (regularized to be
stable within a trajectory), joined only in the reconstruction decoder
$D_\psi(\zvec_t,\uvec_t)$. Theorem~\ref{thm:lrc} shows that this split
structurally supplies the ingredient of optimal representation theory in~\ref{thm:criterion}.
 
\begin{theorem}[Optimality of LRC-JEPA]
\label{thm:lrc}
Under Assumptions~\ref{ass:surrogate}, \ref{ass:2gauss},
and~\ref{ass:lrc} (two-factor Gaussian world $X_t=g(s_t,v)$: hidden state
$s_t\in\Rb^d_s$, trajectory-constant appearance
nuisance $v\perp s$), every global minimizer of $\Lcal_{\mathrm{LRC}}$ satisfies:
(i)~\emph{sufficiency:} $\uvec_t$ is sufficient for $v$, and $\zvec_t$ is a sufficient statistic of $s_t$ conditioned on $u_t$;
(ii)~\emph{minimality:} $I(Z_t;X_t)$ is minimal among prediction-sufficient representations;
(iii)~\emph{invariance:} $I(Z_t;V)=0$; and
(iv)~\emph{disentanglement:} $q(\zvec_t)=\Ncal(0,I_n)$ and $\TC(Z_t)=0$. 
\end{theorem}
Theorem~\ref{thm:lrc} formalizes how two components induce this separation. First, under squared reconstruction loss, information already represented by $U_t$ cannot further reduce the reconstruction residual, so residual-connected reconstruction rewards complementary dynamic information in $Z_t$. Second, detaching $Z_t$ removes this reconstruction constraint and can yield a predictable yet insufficient state. Appendix~\ref{sec:anal_lrc} gives the full analysis and proof.

\section{Results}
\label{sec:results}

We evaluate latent planning, physical-state accessibility, reconstruction interventions, and architectural ablations in simulation and on Bridge-v2 real-world dataset \cite{walke2023bridgedatav2}.

\subsection{Latent Planning Performance}

\paragraph{Evaluation Protocol.}
We evaluate four simulated environments and Bridge-v2, spanning 2D and 3D manipulation and motion planning. Simulation uses image-goal model predictive control (MPC): the cross-entropy method (CEM) optimizes candidate action sequences by terminal latent distance, executes the first segment, and replans. Bridge-v2 uses offline expert-action recovery measured by recall and expert-action rank percentile.

\paragraph{Baseline.}

We compare with LWM~\cite{maes2026leworldmodel}, V-JEPA2~\cite{assran2025vjepa2}, DINO-WM~\cite{zhou2024dinowm}, PLDM~\cite{sobal2025pldm}, C-JEPA~\cite{nam2026causaljepa}, and EB-JEPA~\cite{terver2026ebjepa} using common trajectory splits, observations, actions, and goals (Appendix~\ref{sec:imp_details}). LWM \cite{maes2026leworldmodel} shares LRC-JEPA's $5.5$M-parameter encoder, predictive-state interface, inference architecture, and active parameter count, isolating the effect of representation disentanglement.

\begin{figure}[ht]
\centering
\includegraphics[width=1.0\textwidth]{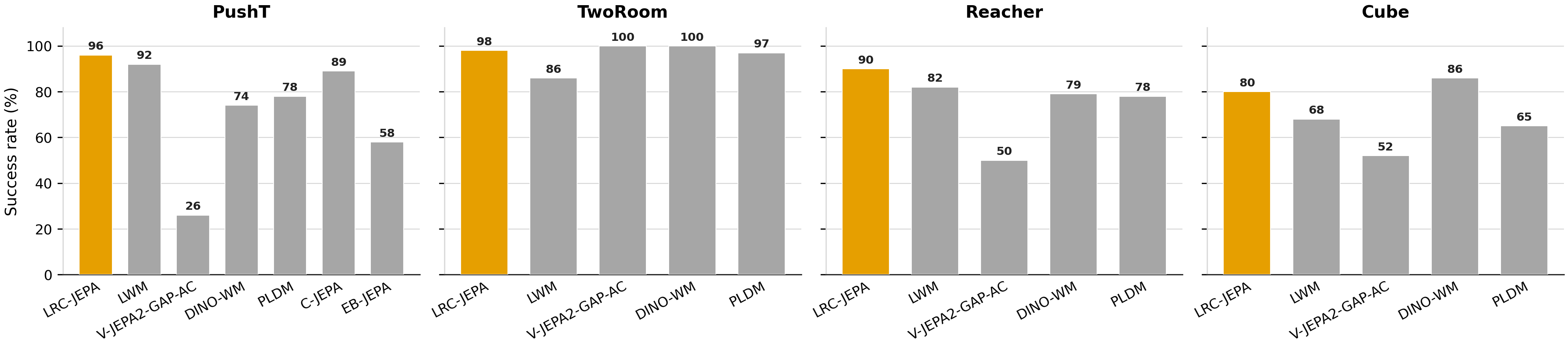}
\caption{CEM planning success rates on four simulated environments. Values are averaged over three repetitions of 50 evaluation scenes.}
\label{fig:latent_planning}
\end{figure}

\begin{figure}[ht]
\centering
\includegraphics[width=1.0\textwidth]{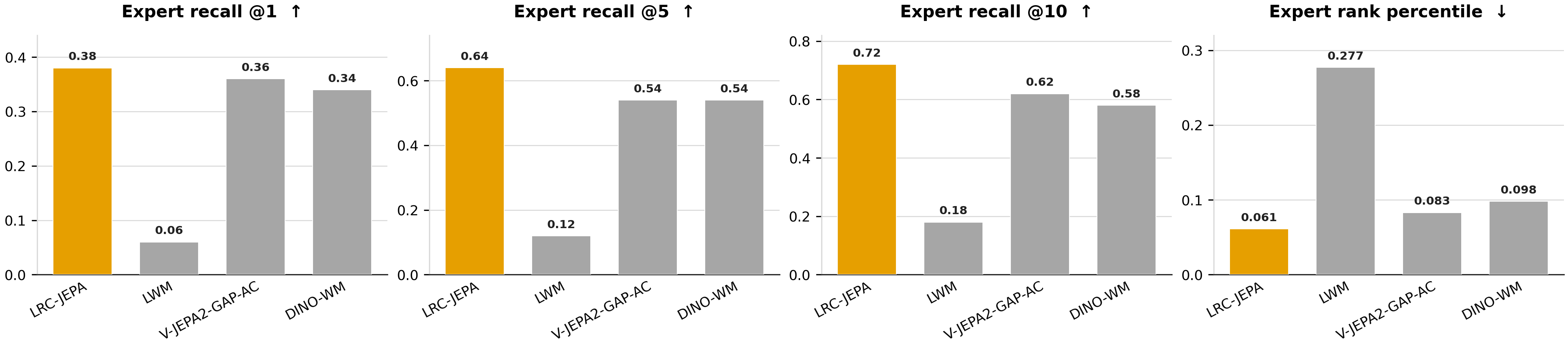}
\caption{Expert action recall and ranking percentile of world models by CEM planning on Bridge-v2 offline dataset. Results averaged on 3 repetitions of 50 distinct scenes.}
\label{fig:offline_latent_planning}
\end{figure}

LRC-JEPA improves average success over parameter-matched LWM by 9pp across the four simulated environments (Fig.~\ref{fig:latent_planning}). It achieves the highest success on kinematics-intensive Push-T and Reacher and remains competitive on TwoRoom and Cube with pretrained DINO-WM, V-JEPA2-GAP-AC, and C-JEPA, despite training its encoder from scratch and planning by one compact state.

On Bridge-v2, LRC-JEPA leads all four offline metrics (Fig.~\ref{fig:offline_latent_planning}), with recall@1/5/10 of $0.38/0.64/0.72$, versus $0.36/0.54/0.62$ for V-JEPA2-GAP-AC and $0.34/0.54/0.58$ for DINO-WM. Its active visual encoder has only $5.5$M parameters, compared with $303.9$M and $22.1$M, respectively (Table~\ref{tab:num_params}, with full planning time reported). Parameter-matched LWM reaches only $0.06/0.12/0.18$, supporting functional routing as the source of the gain rather than model scale.

\begin{table}[H]
\caption{The number of active parameters of each world model encoder and the average full planning time over 50 different trajectories.}
\label{tab:num_params}
\small
\centering
\setlength{\tabcolsep}{4pt}
\begin{tabular}{lccccc}
\toprule
\textbf{Model} & LRC-JEPA & LWM & V-JEPA2-GAP-AC & DINO-WM & EB-JEPA \\
\midrule
\textbf{Num. of params.} & 5.5M & 5.5M & 303.9M & 22.1M & 1.1M \\
\textbf{Full planning time (s)} & 0.45 & 0.48 & 0.83 & 57.56 & 0.20 \\
\bottomrule
\end{tabular}
\end{table}

\subsection{Physical Quantity Probing}
\label{sec:probe_results}

\paragraph{Evaluation Protocol.}
We assess physical information accessible from LRC-JEPA's predictive latent $\mathbf{z}$ and context embedding $\mathbf{u}$ using ridge regression and MLP probes. Probes are trained on frozen, converged encoder outputs and evaluated on held-out data. Targets are simulator-provided physical quantities for simulated environments and pseudo-labels derived from SAM3 \cite{carion2025sam3} object masks and attributes for Bridge-v2. We evaluate LWM's single per-frame latent using the same protocol; fitting details appear in Appendix~\ref{sec:imp_details}.

\paragraph{Comparison with baseline.}
LRC-JEPA outperforms LWM in 28 of 32 Cube comparisons (Table~\ref{tab:cube_probe_comparison}), with the largest gains on kinematic and control-relevant quantities. For joint velocity, linear and MLP correlations rise from $0.111$ and $0.068$ to $0.392$ and $0.312$, respectively. Push-T and TwoRoom results appear in Appendix Tables~\ref{tab:pusht_probe_comparison} and~\ref{tab:tworoom_probe_comparison}.

\begin{table}[ht]
\centering
\caption{Physical quantity probes on Cube.
Each entry reports MSE ($\downarrow$) / Pearson correlation
$r$ ($\uparrow$).
\textbf{Bold} marks the better predictive latent between LWM and
LRC-JEPA; \underline{underlining} marks the better representation between
LRC-JEPA's $\mathbf{z}$ and $\mathbf{u}$, separately for each
metric and probe type.
EE: end-effector; quat.: quaternion.}
\label{tab:cube_probe_comparison}
\label{tab:cube_probe}
\footnotesize
\setlength{\tabcolsep}{2pt}
\renewcommand{\arraystretch}{1.12}
\begin{tabular*}{\linewidth}{
  @{\extracolsep{\fill}}lcccccc@{}
}
\toprule
\multirow{3}{*}{\textbf{Quantity}}
& \multicolumn{3}{c}{\textbf{Linear}}
& \multicolumn{3}{c}{\textbf{MLP}} \\
\cmidrule(lr){2-4}\cmidrule(lr){5-7}
& LWM
& \multicolumn{2}{c}{LRC-JEPA}
& LWM
& \multicolumn{2}{c}{LRC-JEPA} \\
\cmidrule(lr){3-4}\cmidrule(lr){6-7}
& $\mathbf{z}$ & $\mathbf{z}$ & $\mathbf{u}$
& $\mathbf{z}$ & $\mathbf{z}$ & $\mathbf{u}$ \\
\midrule
Joint pos.
& 0.293 / 0.698
& \textbf{\underline{0.274}} / \textbf{\underline{0.723}}
& 0.572 / 0.500
& 0.311 / 0.692
& \textbf{\underline{0.305}} / \textbf{\underline{0.696}}
& 0.444 / 0.653 \\

Joint vel.
& 0.969 / 0.111
& \textbf{\underline{0.734}} / \textbf{\underline{0.392}}
& 0.990 / 0.039
& 0.999 / 0.068
& \textbf{\underline{0.893}} / \textbf{\underline{0.312}}
& 1.013 / 0.026 \\

EE pos.
& 0.019 / 0.990
& \textbf{\underline{0.011}} / \textbf{\underline{0.995}}
& 0.380 / 0.685
& 0.009 / 0.995
& \textbf{\underline{0.007}} / \textbf{\underline{0.997}}
& 0.144 / 0.925 \\

EE yaw
& 0.989 / 0.073
& \textbf{0.930} / \textbf{0.250}
& \underline{0.876} / \underline{0.336}
& \textbf{1.009} / 0.100
& 1.027 / \textbf{0.128}
& \underline{0.897} / \underline{0.321} \\

Gripper
& 0.077 / 0.960
& \textbf{\underline{0.061}} / \textbf{\underline{0.969}}
& 0.193 / 0.897
& 0.038 / 0.981
& \textbf{\underline{0.035}} / \textbf{\underline{0.982}}
& 0.084 / 0.957 \\

Block pos.
& 0.007 / 0.996
& \textbf{\underline{0.006}} / \textbf{\underline{0.997}}
& 0.120 / 0.933
& 0.006 / 0.997
& \textbf{\underline{0.005}} / \textbf{\underline{0.998}}
& 0.017 / 0.991 \\

Block quat.
& 0.678 / 0.053
& \textbf{\underline{0.653}} / \textbf{\underline{0.164}}
& 0.672 / 0.073
& 0.694 / 0.048
& \textbf{\underline{0.692}} / \textbf{0.063}
& 0.694 / \underline{0.094} \\

Block yaw
& \textbf{1.010} / 0.056
& 1.012 / \textbf{0.060}
& \underline{0.971} / \underline{0.189}
& \textbf{1.046} / \textbf{0.066}
& 1.052 / 0.065
& \underline{0.968} / \underline{0.230} \\
\bottomrule
\end{tabular*}
\end{table}

\paragraph{Information Allocation Across Streams.}
Table~\ref{tab:cube_probe} reveals the intended allocation: $\mathbf{z}$ better predicts joint position and velocity, end-effector and block positions, and gripper state, whereas $\mathbf{u}$ better predicts the more slowly varying end-effector and block yaw. Concatenation balances the two on most quantity probes (Appendix Table~\ref{tab:cube_probe_z_plus_u}). These results support a bias toward control-relevant kinematics in $\mathbf{z}$ and persistent context in $\mathbf{u}$.

On Bridge-v2, SAM3-derived masks identify gripper, moving, static, and background objects for centroid and RGB probes. Predictive $\mathbf{z}$ better captures gripper and moving-object centroids, while context $\mathbf{u}$ generally better captures RGB values (Appendix Table~\ref{tab:bridge_probe}), again separating motion-relevant information from persistent appearance.

\subsection{Reconstruction Interventions}

\begin{figure}[ht]
\centering
\includegraphics[width=1.0\textwidth]{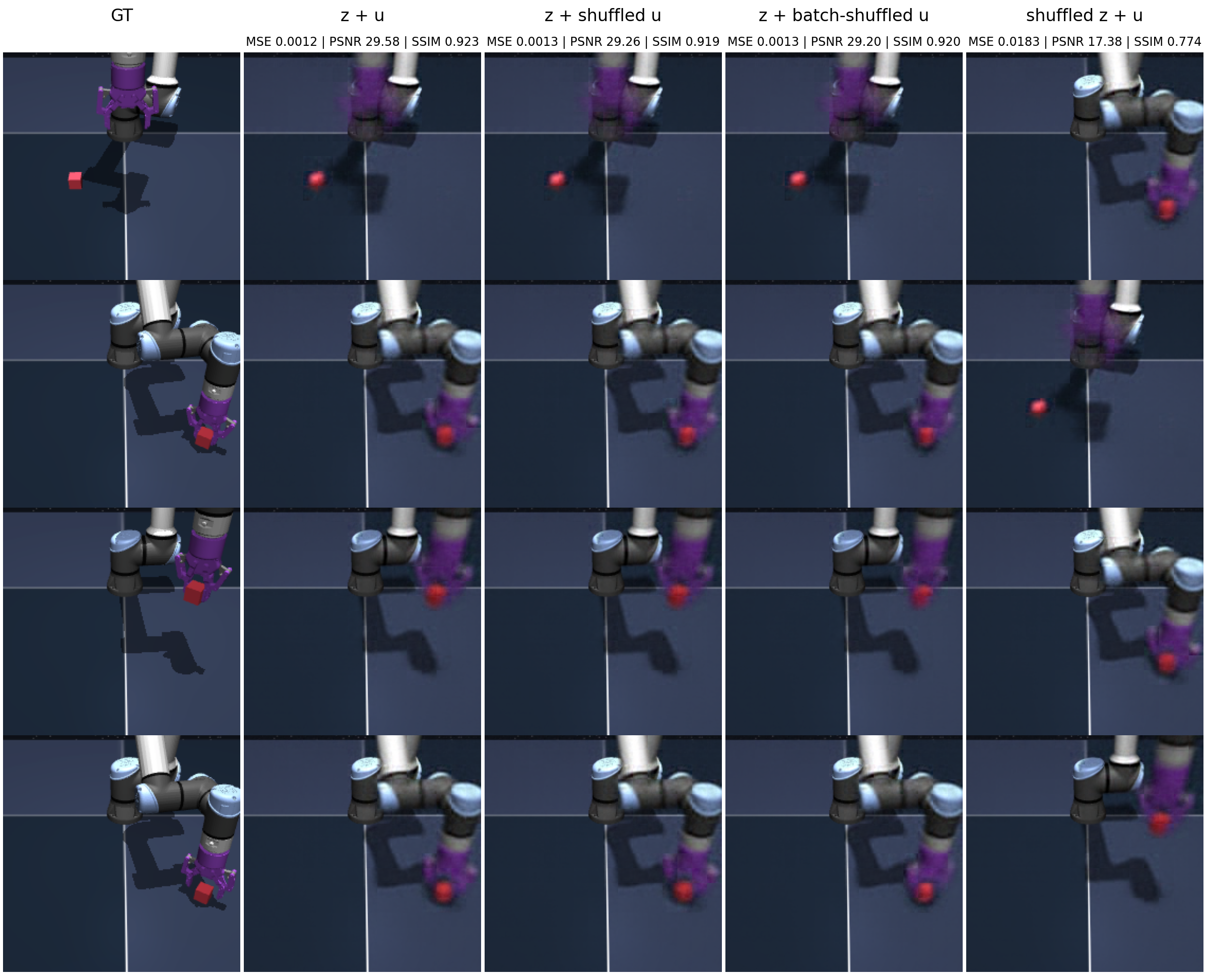}
\caption{Reconstruction interventions on Cube environment. ``Shuffled'' denotes
        within-trajectory permutation, while ``batch-shuffled'' denotes
        replacement using context codes from other trajectories. Metrics calculated on the shown sample.}
\label{fig:cube}
\end{figure}

The training-only decoder makes information allocation interpretable. We reconstruct unaltered latent pairs, shuffle either stream within a trajectory, and replace $\mathbf{u}$ across trajectories (batch shuffling).

On Cube (Fig.~\ref{fig:cube}), within-trajectory $\mathbf{u}$ shuffling causes little degradation (PSNR $29.26$ vs. $29.58$; SSIM $0.919$ vs. $0.923$), consistent with temporal invariance. Cross-trajectory $\mathbf{u}$ replacement changes slowly varying gripper and block orientations, whereas shuffling $\mathbf{z}$ changes arm, gripper, and block locations and yields the lowest quality (PSNR $17.38$, SSIM $0.774$). These interventions corroborate the probes: $\mathbf{z}$ captures kinematics and $\mathbf{u}$ slowly varying context. Appendix Figs.~\ref{fig:pusht} and~\ref{fig:tworoom} show further examples on Push-T and TwoRoom.

\section{Ablation Study}
\label{sec:ablation}

\begin{figure}[t]
\centering
\includegraphics[width=\textwidth]{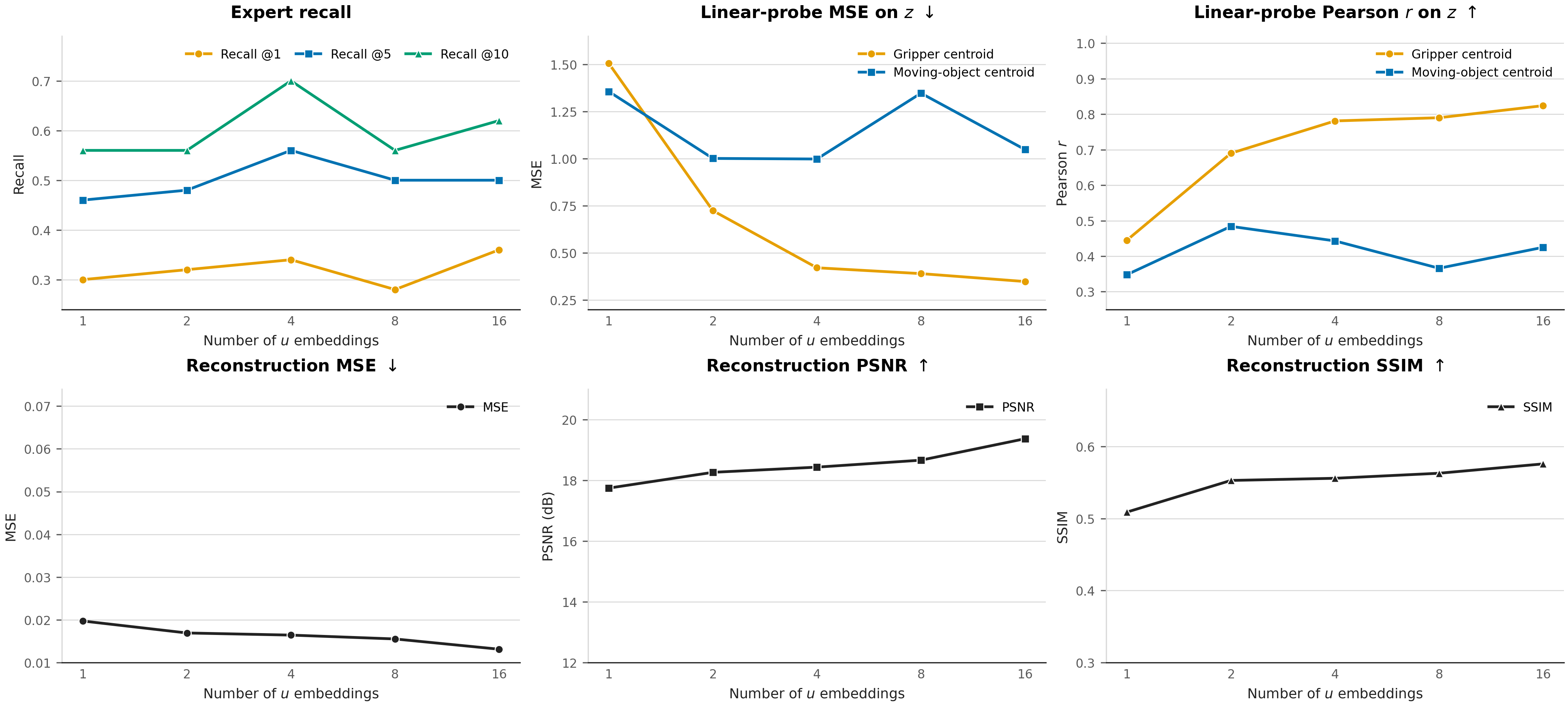}
\caption{Effect of the number of residual-context queries $Q$ on Bridge-v2. We report expert-action recall, linear probes on $\mathbf{z}$, and reconstruction MSE, PSNR, and SSIM.}
\label{fig:num_u_ablation}
\end{figure}

Varying the number of residual-context queries $Q$ from 1 to 16 on Bridge-v2 primarily affects reconstruction: MSE improves monotonically, as do PSNR and SSIM (Fig.~\ref{fig:num_u_ablation}). Planning recall and $\mathbf{z}$ probes vary less consistently, although gripper-centroid probing improves with $Q$. Thus a small context bottleneck captures useful persistent information, while additional queries mainly improve reconstruction.

Furthermore, we train LRC-JEPA and its stop-gradient variant on the push-T and Bridge-v2 environments by applying a stop gradient operation on the residual connection from $z$ to the decoder. Table~\ref{tab:sg_ablation} shows that blocking this gradient reduces Push-T success from $96\%$ to $92\%$. The effect is substantially larger on Bridge-v2: recall@1/5/10 decreases from $0.38/0.64/0.72$ to $0.10/0.28/0.32$, while the rank percentile worsens from $0.061$ to $0.286$. Thus, affected by the entanglement of contextual nuisance, next latent prediction and SIGReg regularization alone do not recover the same planning-relevant representation, particularly in the visually diverse real-world setting. The residual connection and reconstruction gradients enforce separation of context from the predictive latent, achieving better planning and control. Appendix~\ref{sec:anal_lrc} provides the corresponding theoretical analysis.

\begin{table}[H]
\caption{
    Effect of blocking reconstruction gradients at $\mathbf{z}$.
    Push-T reports planning success in percent. Bridge-v2 reports
    expert-action recall and rank percentile (pct.).
}
\label{tab:sg_ablation}
\centering
\footnotesize
\setlength{\tabcolsep}{4pt}
\begin{tabular}{lccc}
\toprule
\textbf{Model}
& \textbf{Push-T $\uparrow$}
& \textbf{Bridge recall@1/5/10 $\uparrow$}
& \textbf{Bridge rank pct. $\downarrow$} \\
\midrule
LRC-JEPA & 96.0 & 0.38/0.64/0.72 & 0.061 \\
w/ stop-gradient & 92.0 & 0.10/0.28/0.32 & 0.286 \\
\bottomrule
\end{tabular}
\end{table}

\section{Conclusion and Limitations}
LRC-JEPA resolves a capacity-allocation problem in compact JEPA world models by routing action-conditioned dynamics through a predictive state $\mathbf{z}$ and persistent visual context through reconstruction-only embeddings $\mathbf{u}$. When competing with a substantially larger pretrained model (4$\times$ for DINO-WM), LRC-JEPA demonstrates the efficiency of single-state latent planning (120$\times$ acceleration) and simultaneously realizes 38\% control performance improvement (in terms of expert action rank percentile) on the real-world Bridge-v2 dataset. Theory, physical-state probes, and reconstruction interventions consistently support the intended functional separation.

Nevertheless, the approach assumes that nuisance context is approximately stable over the training horizon; rapidly changing but uncontrollable or irrelevant factors may be routed imperfectly. The decoder and context regularizer also add training-time computation and introduce sensitivity to bottleneck capacity and loss weights, although neither is used for planning. Finally, real-world data is evaluated through offline action recovery rather than closed-loop robot execution. Testing longer horizons, dynamic visual nuisances, and real-robot control is therefore an important direction for future work.

\subsection*{AI use statement}
During the manuscript preparation, we used generative AI tools to assist with manuscript revision, review, and reference checking and formatting. We also used generative AI for code and documentation editing,
feedback on experimental methodology and result interpretation.
The authors are responsible for independently verifying AI-assisted text, code, citations, and mathematical claims and take full responsibility for the final manuscript and released artifacts.

\subsection*{Ethics statement}
This work studies representation learning for visual world models using
simulated control tasks and offline robot data. The Bridge-v2 evaluation
measures offline action ranking and does not establish the safety or
reliability of deployment on physical robots. Applying these models to
physical systems requires additional validation and appropriate operational
safeguards. Use of the datasets, third-party software, and pretrained models
is subject to their respective licenses and access conditions. 

\subsection*{Reproducibility statement}
Appendix~\ref{sec:imp_details} describes the model architectures, training settings, data
preprocessing, and evaluation protocols. 

\begingroup
\small
\raggedright
\bibliographystyle{plainnat}
\bibliography{references}

\begin{thebibliography}{41}
\providecommand{\natexlab}[1]{#1}
\providecommand{\url}[1]{\texttt{#1}}
\expandafter\ifx\csname urlstyle\endcsname\relax
  \providecommand{\doi}[1]{doi: #1}\else
  \providecommand{\doi}{doi: \begingroup \urlstyle{rm}\Url}\fi

\bibitem[Achille and Soatto(2018{\natexlab{a}})]{achille2018emergence}
Alessandro Achille and Stefano Soatto.
\newblock Emergence of invariance and disentanglement in deep representations.
\newblock \emph{Journal of Machine Learning Research}, 19\penalty0
  (50):\penalty0 1--34, 2018{\natexlab{a}}.
\newblock URL \url{https://jmlr.org/papers/v19/17-646.html}.

\bibitem[Achille and Soatto(2018{\natexlab{b}})]{achille2018information}
Alessandro Achille and Stefano Soatto.
\newblock Information dropout: Learning optimal representations through noisy
  computation.
\newblock \emph{IEEE Transactions on Pattern Analysis and Machine
  Intelligence}, 40\penalty0 (12):\penalty0 2897--2905, 2018{\natexlab{b}}.
\newblock \doi{10.1109/TPAMI.2017.2784440}.
\newblock URL \url{https://doi.org/10.1109/TPAMI.2017.2784440}.

\bibitem[Alemi et~al.(2017)Alemi, Fischer, Dillon, and Murphy]{alemi2017deep}
Alexander~A. Alemi, Ian Fischer, Joshua~V. Dillon, and Kevin Murphy.
\newblock Deep variational information bottleneck.
\newblock In \emph{International Conference on Learning Representations}, 2017.
\newblock URL \url{https://arxiv.org/abs/1612.00410}.

\bibitem[Assran et~al.(2023)Assran, Duval, Misra, Bojanowski, Vincent, Rabbat,
  LeCun, and Ballas]{assran2023ijepa}
Mahmoud Assran, Quentin Duval, Ishan Misra, Piotr Bojanowski, Pascal Vincent,
  Michael Rabbat, Yann LeCun, and Nicolas Ballas.
\newblock Self-supervised learning from images with a joint-embedding
  predictive architecture.
\newblock In \emph{Proceedings of the IEEE/CVF Conference on Computer Vision
  and Pattern Recognition}, pages 15619--15629, 2023.
\newblock URL
  \url{https://openaccess.thecvf.com/content/CVPR2023/html/Assran_Self-Supervised_Learning_From_Images_With_a_Joint-Embedding_Predictive_Architecture_CVPR_2023_paper.html}.

\bibitem[Assran et~al.(2025)Assran, Bardes, Fan, Garrido, Howes, Komeili,
  Muckley, Rizvi, Roberts, Sinha, Zholus, Arnaud, Gejji, Martin, Hogan, Dugas,
  Bojanowski, Khalidov, Labatut, Massa, Szafraniec, Krishnakumar, Li, Ma,
  Chandar, Meier, LeCun, Rabbat, and Ballas]{assran2025vjepa2}
Mido Assran, Adrien Bardes, David Fan, Quentin Garrido, Russell Howes, Mojtaba
  Komeili, Matthew Muckley, Ammar Rizvi, Claire Roberts, Koustuv Sinha, Artem
  Zholus, Sergio Arnaud, Abha Gejji, Ada Martin, Francois~Robert Hogan, Daniel
  Dugas, Piotr Bojanowski, Vasil Khalidov, Patrick Labatut, Francisco Massa,
  Marc Szafraniec, Kapil Krishnakumar, Yong Li, Xiaodong Ma, Sarath Chandar,
  Franziska Meier, Yann LeCun, Michael Rabbat, and Nicolas Ballas.
\newblock {V-JEPA 2}: Self-supervised video models enable understanding,
  prediction and planning.
\newblock \emph{arXiv preprint arXiv:2506.09985}, 2025.
\newblock URL \url{https://arxiv.org/abs/2506.09985}.

\bibitem[Balestriero and LeCun(2025)]{balestriero2025lejepa}
Randall Balestriero and Yann LeCun.
\newblock {LeJEPA}: Provable and scalable self-supervised learning without the
  heuristics.
\newblock \emph{arXiv preprint arXiv:2511.08544}, 2025.
\newblock URL \url{https://arxiv.org/abs/2511.08544}.

\bibitem[Bardes et~al.(2022)Bardes, Ponce, and LeCun]{bardes2022vicreg}
Adrien Bardes, Jean Ponce, and Yann LeCun.
\newblock {VICReg}: Variance-invariance-covariance regularization for
  self-supervised learning.
\newblock In \emph{International Conference on Learning Representations}, 2022.
\newblock URL \url{https://openreview.net/forum?id=xm6YD62D1Ub}.

\bibitem[Bardes et~al.(2023)Bardes, Ponce, and LeCun]{bardes2023mcjepa}
Adrien Bardes, Jean Ponce, and Yann LeCun.
\newblock {MC-JEPA}: A joint-embedding predictive architecture for
  self-supervised learning of motion and content features.
\newblock \emph{arXiv preprint arXiv:2307.12698}, 2023.
\newblock URL \url{https://arxiv.org/abs/2307.12698}.

\bibitem[Bardes et~al.(2024)Bardes, Garrido, Ponce, Chen, Rabbat, LeCun,
  Assran, and Ballas]{bardes2024vjepa}
Adrien Bardes, Quentin Garrido, Jean Ponce, Xinlei Chen, Michael Rabbat, Yann
  LeCun, Mido Assran, and Nicolas Ballas.
\newblock Revisiting feature prediction for learning visual representations
  from video.
\newblock \emph{Transactions on Machine Learning Research}, 2024.
\newblock URL \url{https://openreview.net/forum?id=QaCCuDfBk2}.

\bibitem[Carion et~al.(2025)Carion, Gustafson, Hu, Debnath, Hu, Suris, Ryali,
  Alwala, Khedr, Huang, Lei, Ma, Guo, Kalla, Marks, Greer, Wang, Sun,
  R{\"a}dle, Afouras, Mavroudi, Xu, Wu, Zhou, Momeni, Hazra, Ding, Vaze,
  Porcher, Li, Li, Kamath, Cheng, Doll{\'a}r, Ravi, Saenko, Zhang, and
  Feichtenhofer]{carion2025sam3}
Nicolas Carion, Laura Gustafson, Yuan-Ting Hu, Shoubhik Debnath, Ronghang Hu,
  Didac Suris, Chaitanya Ryali, Kalyan~Vasudev Alwala, Haitham Khedr, Andrew
  Huang, Jie Lei, Tengyu Ma, Baishan Guo, Arpit Kalla, Markus Marks, Joseph
  Greer, Meng Wang, Peize Sun, Roman R{\"a}dle, Triantafyllos Afouras,
  Effrosyni Mavroudi, Katherine Xu, Tsung-Han Wu, Yu~Zhou, Liliane Momeni,
  Rishi Hazra, Shuangrui Ding, Sagar Vaze, Francois Porcher, Feng Li, Siyuan
  Li, Aishwarya Kamath, Ho~Kei Cheng, Piotr Doll{\'a}r, Nikhila Ravi, Kate
  Saenko, Pengchuan Zhang, and Christoph Feichtenhofer.
\newblock {SAM 3}: Segment anything with concepts.
\newblock \emph{arXiv preprint arXiv:2511.16719}, 2025.
\newblock URL \url{https://arxiv.org/abs/2511.16719}.

\bibitem[Chen et~al.(2018)Chen, Li, Grosse, and Duvenaud]{chen2018isolating}
Ricky T.~Q. Chen, Xuechen Li, Roger Grosse, and David Duvenaud.
\newblock Isolating sources of disentanglement in variational autoencoders.
\newblock In \emph{Advances in Neural Information Processing Systems},
  volume~31, 2018.
\newblock URL \url{https://arxiv.org/abs/1802.04942}.

\bibitem[Cui et~al.(2026)Cui, Zhang, Wen, and Wang]{cui2026generalization}
Jingyi Cui, Qi~Zhang, Hongwei Wen, and Yisen Wang.
\newblock A generalization theory for {JEPA}-based world models.
\newblock \emph{arXiv preprint arXiv:2606.27014}, 2026.
\newblock URL \url{https://arxiv.org/abs/2606.27014}.

\bibitem[Denton and Birodkar(2017)]{denton2017disentangled}
Emily~L. Denton and Vighnesh Birodkar.
\newblock Unsupervised learning of disentangled representations from video.
\newblock In \emph{Advances in Neural Information Processing Systems},
  volume~30, 2017.
\newblock URL
  \url{https://proceedings.neurips.cc/paper/2017/hash/2d2ca7eedf739ef4c3800713ec482e1a-Abstract.html}.

\bibitem[Dosovitskiy et~al.(2021)Dosovitskiy, Beyer, Kolesnikov, Weissenborn,
  Zhai, Unterthiner, Dehghani, Minderer, Heigold, Gelly, Uszkoreit, and
  Houlsby]{dosovitskiy2020image}
Alexey Dosovitskiy, Lucas Beyer, Alexander Kolesnikov, Dirk Weissenborn,
  Xiaohua Zhai, Thomas Unterthiner, Mostafa Dehghani, Matthias Minderer, Georg
  Heigold, Sylvain Gelly, Jakob Uszkoreit, and Neil Houlsby.
\newblock An image is worth {16x16} words: Transformers for image recognition
  at scale.
\newblock In \emph{International Conference on Learning Representations}, 2021.
\newblock URL \url{https://arxiv.org/abs/2010.11929}.

\bibitem[Ebert et~al.(2018)Ebert, Finn, Dasari, Xie, Lee, and
  Levine]{ebert2018visualforesight}
Frederik Ebert, Chelsea Finn, Sudeep Dasari, Annie Xie, Alex Lee, and Sergey
  Levine.
\newblock Visual foresight: Model-based deep reinforcement learning for
  vision-based robotic control.
\newblock \emph{arXiv preprint arXiv:1812.00568}, 2018.
\newblock URL \url{https://arxiv.org/abs/1812.00568}.

\bibitem[Finn et~al.(2016)Finn, Goodfellow, and Levine]{finn2016unsupervised}
Chelsea Finn, Ian Goodfellow, and Sergey Levine.
\newblock Unsupervised learning for physical interaction through video
  prediction.
\newblock In \emph{Advances in Neural Information Processing Systems},
  volume~29, 2016.
\newblock URL
  \url{https://proceedings.neurips.cc/paper/2016/hash/d9d4f495e875a2e075a1a4a6e1b9770f-Abstract.html}.

\bibitem[Fu et~al.(2025)Fu, Lian, Wang, Shi, Wang, Yala, Darrell, Efros, and
  Goldberg]{fu2024crossmae}
Letian Fu, Long Lian, Renhao Wang, Baifeng Shi, Xudong Wang, Adam Yala, Trevor
  Darrell, Alexei~A. Efros, and Ken Goldberg.
\newblock Rethinking patch dependence for masked autoencoders.
\newblock \emph{Transactions on Machine Learning Research}, 2025.
\newblock URL \url{https://openreview.net/forum?id=JT2KMuo2BV}.

\bibitem[Ha and Schmidhuber(2018)]{ha2018worldmodels}
David Ha and J{\"u}rgen Schmidhuber.
\newblock World models.
\newblock \emph{arXiv preprint arXiv:1803.10122}, 2018.
\newblock URL \url{https://arxiv.org/abs/1803.10122}.

\bibitem[Hafner et~al.(2019)Hafner, Lillicrap, Fischer, Villegas, Ha, Lee, and
  Davidson]{hafner2019planet}
Danijar Hafner, Timothy Lillicrap, Ian Fischer, Ruben Villegas, David Ha,
  Honglak Lee, and James Davidson.
\newblock Learning latent dynamics for planning from pixels.
\newblock In \emph{Proceedings of the 36th International Conference on Machine
  Learning}, volume~97 of \emph{Proceedings of Machine Learning Research},
  pages 2555--2565. PMLR, 2019.
\newblock URL \url{https://proceedings.mlr.press/v97/hafner19a.html}.

\bibitem[Hafner et~al.(2025)Hafner, Pasukonis, Ba, and
  Lillicrap]{hafner2023dreamerv3}
Danijar Hafner, Jurgis Pasukonis, Jimmy Ba, and Timothy Lillicrap.
\newblock Mastering diverse control tasks through world models.
\newblock \emph{Nature}, 640:\penalty0 647--653, 2025.
\newblock \doi{10.1038/s41586-025-08744-2}.
\newblock URL \url{https://doi.org/10.1038/s41586-025-08744-2}.

\bibitem[Hansen et~al.(2024)Hansen, Su, and Wang]{hansen2024tdmpc2}
Nicklas Hansen, Hao Su, and Xiaolong Wang.
\newblock {TD-MPC2}: Scalable, robust world models for continuous control.
\newblock In \emph{International Conference on Learning Representations}, 2024.
\newblock URL \url{https://openreview.net/forum?id=Oxh5CstDJU}.

\bibitem[Hansen et~al.(2022)Hansen, Su, and Wang]{hansen2022tdmpc}
Nicklas~A. Hansen, Hao Su, and Xiaolong Wang.
\newblock Temporal difference learning for model predictive control.
\newblock In \emph{Proceedings of the 39th International Conference on Machine
  Learning}, volume 162 of \emph{Proceedings of Machine Learning Research},
  pages 8387--8406. PMLR, 2022.
\newblock URL \url{https://proceedings.mlr.press/v162/hansen22a.html}.

\bibitem[Hsieh et~al.(2018)Hsieh, Liu, Huang, Fei-Fei, and
  Niebles]{hsieh2018ddpae}
Jun-Ting Hsieh, Bingbin Liu, De-An Huang, Li~Fei-Fei, and Juan~Carlos Niebles.
\newblock Learning to decompose and disentangle representations for video
  prediction.
\newblock In \emph{Advances in Neural Information Processing Systems},
  volume~31, 2018.
\newblock URL
  \url{https://proceedings.neurips.cc/paper/2018/hash/496e05e1aea0a9c4655800e8a7b9ea28-Abstract.html}.

\bibitem[Klindt et~al.(2026)Klindt, LeCun, and Balestriero]{klindt2026lejepa}
David Klindt, Yann LeCun, and Randall Balestriero.
\newblock When does {LeJEPA} learn a world model?
\newblock \emph{arXiv preprint arXiv:2605.26379}, 2026.
\newblock URL \url{https://arxiv.org/abs/2605.26379}.

\bibitem[Liu et~al.(2026)Liu, Suo, Jin, Ping, Iwasawa, Matsuo, and
  Zhu]{liu2026temporallycentered}
Chang Liu, Fei Suo, Yanzhou Jin, Zeyu Ping, Yusuke Iwasawa, Yutaka Matsuo, and
  Yaonan Zhu.
\newblock Temporally centered {SIGReg} improves {LeWorldModel} representations
  for robot policy learning.
\newblock \emph{arXiv preprint arXiv:2607.26924}, 2026.
\newblock URL \url{https://arxiv.org/abs/2607.26924}.

\bibitem[Maes et~al.(2026)Maes, Le~Lidec, Scieur, LeCun, and
  Balestriero]{maes2026leworldmodel}
Lucas Maes, Quentin Le~Lidec, Damien Scieur, Yann LeCun, and Randall
  Balestriero.
\newblock {LeWorldModel}: Stable end-to-end joint-embedding predictive
  architecture from pixels.
\newblock \emph{arXiv preprint arXiv:2603.19312}, 2026.
\newblock URL \url{https://arxiv.org/abs/2603.19312}.

\bibitem[Nam et~al.(2026)Nam, Le~Lidec, Maes, LeCun, and
  Balestriero]{nam2026causaljepa}
Heejeong Nam, Quentin Le~Lidec, Lucas Maes, Yann LeCun, and Randall
  Balestriero.
\newblock {Causal-JEPA}: Learning world models through object-level latent
  masking.
\newblock In \emph{Proceedings of the 43rd International Conference on Machine
  Learning}, 2026.
\newblock URL \url{https://arxiv.org/abs/2602.11389}.

\bibitem[{NVIDIA}(2025)]{nvidia2025cosmos}
{NVIDIA}.
\newblock Cosmos world foundation model platform for physical {AI}.
\newblock \emph{arXiv preprint arXiv:2501.03575}, 2025.
\newblock URL \url{https://arxiv.org/abs/2501.03575}.

\bibitem[Oquab et~al.(2024)Oquab, Darcet, Moutakanni, Vo, Szafraniec, Khalidov,
  Fernandez, Haziza, Massa, El-Nouby, Assran, Ballas, Galuba, Howes, Huang, Li,
  Misra, Rabbat, Sharma, Synnaeve, Xu, J{\'e}gou, Mairal, Labatut, Joulin, and
  Bojanowski]{oquab2024dinov2}
Maxime Oquab, Timoth{\'e}e Darcet, Th{\'e}o Moutakanni, Huy~V. Vo, Marc
  Szafraniec, Vasil Khalidov, Pierre Fernandez, Daniel Haziza, Francisco Massa,
  Alaaeldin El-Nouby, Mido Assran, Nicolas Ballas, Wojciech Galuba, Russell
  Howes, Po-Yao Huang, Shang-Wen Li, Ishan Misra, Michael Rabbat, Vasu Sharma,
  Gabriel Synnaeve, Hu~Xu, Herv{\'e} J{\'e}gou, Julien Mairal, Patrick Labatut,
  Armand Joulin, and Piotr Bojanowski.
\newblock {DINOv2}: Learning robust visual features without supervision.
\newblock \emph{Transactions on Machine Learning Research}, 2024.
\newblock URL \url{https://openreview.net/forum?id=a68SUt6zFt}.

\bibitem[Pan et~al.(2022)Pan, Zhu, Wang, and Yang]{pan2022isodream}
Minting Pan, Xiangming Zhu, Yunbo Wang, and Xiaokang Yang.
\newblock {Iso-Dream}: Isolating and leveraging noncontrollable visual dynamics
  in world models.
\newblock In \emph{Advances in Neural Information Processing Systems},
  volume~35, pages 23178--23191, 2022.
\newblock URL
  \url{https://proceedings.neurips.cc/paper_files/paper/2022/hash/9316769afaaeeaad42a9e3633b14e801-Abstract-Conference.html}.

\bibitem[Peebles and Xie(2023)]{peebles2023scalable}
William Peebles and Saining Xie.
\newblock Scalable diffusion models with transformers.
\newblock In \emph{Proceedings of the IEEE/CVF International Conference on
  Computer Vision}, pages 4172--4182, 2023.
\newblock URL \url{https://arxiv.org/abs/2212.09748}.

\bibitem[Sobal et~al.(2025)Sobal, Zhang, Cho, Balestriero, Rudner, and
  LeCun]{sobal2025pldm}
Uladzislau Sobal, Wancong Zhang, Kyunghyun Cho, Randall Balestriero, Tim G.~J.
  Rudner, and Yann LeCun.
\newblock Learning from reward-free offline data: A case for planning with
  latent dynamics models.
\newblock In \emph{Advances in Neural Information Processing Systems},
  volume~38, 2025.
\newblock URL
  \url{https://proceedings.neurips.cc/paper_files/paper/2025/hash/3e7cf447f21cd11c846463affefce665-Abstract-Conference.html}.

\bibitem[Sobal et~al.(2022)Sobal, S~V, Jalagam, Carion, Cho, and
  LeCun]{sobal2022slowfeatures}
Vlad Sobal, Jyothir S~V, Siddhartha Jalagam, Nicolas Carion, Kyunghyun Cho, and
  Yann LeCun.
\newblock Joint embedding predictive architectures focus on slow features.
\newblock \emph{arXiv preprint arXiv:2211.10831}, 2022.
\newblock URL \url{https://arxiv.org/abs/2211.10831}.

\bibitem[Terver et~al.(2026)Terver, Balestriero, Dervishi, Fan, Garrido,
  Nagarajan, Sinha, Zhang, Rabbat, LeCun, and Bar]{terver2026ebjepa}
Basile Terver, Randall Balestriero, Megi Dervishi, David Fan, Quentin Garrido,
  Tushar Nagarajan, Koustuv Sinha, Wancong Zhang, Mike Rabbat, Yann LeCun, and
  Amir Bar.
\newblock A lightweight library for energy-based joint-embedding predictive
  architectures.
\newblock \emph{arXiv preprint arXiv:2602.03604}, 2026.
\newblock URL \url{https://arxiv.org/abs/2602.03604}.

\bibitem[Thil et~al.(2026)Thil, Read, Kaddah, and Doquet]{thil2026subspacejepa}
Lucas Thil, Jesse Read, Rim Kaddah, and Guillaume Doquet.
\newblock Subspace-decomposed {JEPAs}: Disentangling progression and content in
  latent world models.
\newblock \emph{arXiv preprint arXiv:2605.31111}, 2026.
\newblock URL \url{https://arxiv.org/abs/2605.31111}.

\bibitem[Tulyakov et~al.(2018)Tulyakov, Liu, Yang, and
  Kautz]{tulyakov2018mocogan}
Sergey Tulyakov, Ming-Yu Liu, Xiaodong Yang, and Jan Kautz.
\newblock {MoCoGAN}: Decomposing motion and content for video generation.
\newblock In \emph{Proceedings of the IEEE Conference on Computer Vision and
  Pattern Recognition}, pages 1526--1535, 2018.
\newblock URL
  \url{https://openaccess.thecvf.com/content_cvpr_2018/html/Tulyakov_MoCoGAN_Decomposing_Motion_CVPR_2018_paper.html}.

\bibitem[Walke et~al.(2023)Walke, Black, Zhao, Vuong, Zheng, Hansen-Estruch,
  He, Myers, Kim, Du, Lee, Fang, Finn, and Levine]{walke2023bridgedatav2}
Homer~Rich Walke, Kevin Black, Tony~Z. Zhao, Quan Vuong, Chongyi Zheng,
  Philippe Hansen-Estruch, Andre~Wang He, Vivek Myers, Moo~Jin Kim, Max Du,
  Abraham Lee, Kuan Fang, Chelsea Finn, and Sergey Levine.
\newblock {BridgeData V2}: A dataset for robot learning at scale.
\newblock In \emph{Proceedings of the 7th Conference on Robot Learning}, volume
  229 of \emph{Proceedings of Machine Learning Research}, pages 1723--1736.
  PMLR, 2023.
\newblock URL \url{https://proceedings.mlr.press/v229/walke23a.html}.

\bibitem[Wu et~al.(2021)Wu, Yao, Wang, and Long]{wu2021motionrnn}
Haixu Wu, Zhiyu Yao, Jianmin Wang, and Mingsheng Long.
\newblock {MotionRNN}: A flexible model for video prediction with
  spacetime-varying motions.
\newblock In \emph{Proceedings of the IEEE/CVF Conference on Computer Vision
  and Pattern Recognition}, pages 15435--15444, 2021.
\newblock URL
  \url{https://openaccess.thecvf.com/content/CVPR2021/html/Wu_MotionRNN_A_Flexible_Model_for_Video_Prediction_With_Spacetime-Varying_Motions_CVPR_2021_paper.html}.

\bibitem[Zhang et~al.(2021)Zhang, McAllister, Calandra, Gal, and
  Levine]{zhang2021bisimulation}
Amy Zhang, Rowan~Thomas McAllister, Roberto Calandra, Yarin Gal, and Sergey
  Levine.
\newblock Learning invariant representations for reinforcement learning without
  reconstruction.
\newblock In \emph{International Conference on Learning Representations}, 2021.
\newblock URL \url{https://openreview.net/forum?id=-2FCwDKRREu}.

\bibitem[Zhang et~al.(2026)Zhang, Guan, Zhang, Li, Zhang, and
  Li]{zhang2026controlled}
Xiangteng Zhang, Yang Guan, Bo~Zhang, Hongyang Li, Ya-Qin Zhang, and
  Shengbo~Eben Li.
\newblock On the identifiability of controlled world models.
\newblock \emph{arXiv preprint arXiv:2607.22430}, 2026.
\newblock URL \url{https://arxiv.org/abs/2607.22430}.

\bibitem[Zhou et~al.(2025)Zhou, Pan, LeCun, and Pinto]{zhou2024dinowm}
Gaoyue Zhou, Hengkai Pan, Yann LeCun, and Lerrel Pinto.
\newblock {DINO-WM}: World models on pre-trained visual features enable
  zero-shot planning.
\newblock In \emph{Proceedings of the 42nd International Conference on Machine
  Learning}, volume 267 of \emph{Proceedings of Machine Learning Research},
  pages 79115--79135. PMLR, 2025.
\newblock URL \url{https://proceedings.mlr.press/v267/zhou25t.html}.

\end{thebibliography}
\endgroup

\clearpage
\appendix

\section{Implementation Details}
\label{sec:imp_details}

\subsection{Neural Networks}
\paragraph{LRC-JEPA architecture.}
Our encoder is a ViT-Tiny trained from scratch with $12$ layers, $3$
attention heads, hidden dimension $192$, and patch size $14$. At
$224\times224$ resolution, it produces $256$ patch tokens and one class
token. Both $d_z$ and $d_u$ are $192$. The number of residual-context
queries $Q$ is varied in ablations. The predictor has $6$ causal
Transformer layers, $16$ attention heads, MLP width $2048$, and dropout
$0.1$. The reconstruction decoder has $3$ cross-attention layers,
hidden dimension $256$, $8$ heads, and MLP width $1024$. The resulting
Bridge-v2 instantiation contains approximately $21.1$ million
parameters including the decoder. However, at planning time, only the predictive state $\mathbf{z}$, action encoder, and dynamics predictor are loaded, resulting in ~5.5M active encoder parameters.

\paragraph{V-JEPA2-GAP-AC baseline.}
V-JEPA2-GAP-AC is our compact action-conditioned realization rather
than the standard V-JEPA2-AC architecture. We freeze a pretrained
V-JEPA2 ViT-L/16 encoder and process images at $256\times256$
resolution. The resulting $256$ spatial patch representations are globally averaged to obtain one
$1024$-dimensional state vector per observation. We then train an action
encoder and a six-layer causal action-conditioned predictor, using the
same three-context/one-target temporal organization as LRC-JEPA. The
predictor uses an internal width of $192$, $16$ attention heads, and MLP
width $2048$, and is trained for $10$ epochs with batch size $192$ and
learning rate $5\times10^{-5}$. This
baseline therefore tests whether a strong pretrained video
representation remains effective when converted into the same compact
global-state planning interface used by LRC-JEPA.

\subsection{Details on Capacity--complexity Stress Tests for Compact JEPA planning}
\label{sec:trade-off}
\paragraph{PCA compression.}
For each checkpoint, we independently fit PCA to 1,000 clean predictive latents $z$ sampled without replacement from the TwoRoom evaluation dataset. At evaluation time, every encoded and autoregressively predicted latent is projected onto the leading \(k\) components as \(\Pi_k(\mathbf{z})=\boldsymbol{\mu}+\mathbf{U}_k\mathbf{U}_k^\top(\mathbf{z}-\boldsymbol{\mu})\), ensuring that current, goal, and predicted representations occupy the same subspace. We evaluate \(k\in\{32,48,64,96,128,192\}\), where \(k=192\) is an exact no-op, on 50 trajectories with three evaluation seeds. The PCA basis is fixed across ranks and repetitions. Logistic trends are fitted to the binomial success counts using \(p(k)=c\,\sigma[\beta(\log_2k-\log_2\kappa)]\); shaded regions show pointwise 95\% intervals from 500 parametric-bootstrap refits. This intervention restricts the effective rank of the latent while retaining its 192-dimensional tensor shape, and therefore measures informational capacity rather than the computational cost of a natively lower-dimensional model.

\paragraph{Static-pattern perturbation.}
For each evaluation trajectory, we sample a trajectory-persistent \(8\times8\) binary grayscale pattern, interpolate it by nearest-neighbor method to \(224\times224\), and broadcast it across the RGB channels. The pattern differs across trajectories but remains fixed over time and is independent of the agent's actions. We blend the same pattern into both current and goal observations according to \(\widetilde{\mathbf{x}}_\alpha=(1-\alpha)\mathbf{x}+\alpha\mathbf{b}\), with \(\alpha\in\{0,0.05,0.10,0.15,0.20,0.30,0.40,0.50\}\). All opacity conditions use the same 50 trajectories and trajectory-specific planning seeds, with \(\alpha=0\) serving as the clean reference. We report the raw success rates with adjacent opacity conditions connected for readability.

\subsection{Details on Physical Quantity Probes}
\label{sec:probe}
\paragraph{Physical quantity probes.}
All probes are trained on frozen representations extracted from individual frames; no encoder parameters are updated. We evaluate the predictive state $\mathbf{z}$, the mean-pooled residual context $\mathbf{u}$, and their concatenation $\mathbf{z}\oplus\mathbf{u}$. For the simulated environments, we randomly sample $10{,}000$ frames and use an $80/20$ train--test split. For Bridge-v2, object properties are obtained from the extracted segmentation masks; one valid moving, static, and background object is selected per frame, retaining $1{,}862$ of $2{,}000$ labeled frames, which are split $80/20$ by episode to prevent trajectory leakage. Each target and representation is evaluated independently using ridge regression with regularization $10^{-3}$ and a two-layer MLP with hidden widths $(256,256)$, ReLU activations, and early stopping. Features and regression targets are standardized for probe fitting. We report target-normalized mean-squared error and Pearson correlation, averaging over dimensions for vector-valued quantities.

\paragraph{Speed evaluation.}
Since our eventual target is closed-loop deployment outside
datacenter-scale infrastructure, inference-speed measurements are
performed on a single consumer-class NVIDIA GeForce RTX~5090 GPU
($32$\,GB). We use the same CEM population, elite count, iteration
count, and horizon for all methods. Full planning time sums every CEM
solve required by a trajectory and includes online current/goal image
encoding, but excludes checkpoint loading, environment stepping, video
generation, and result serialization. Reported values are averaged
over $50$ trajectories.

\subsection{Implementation Details}
\label{sec:training_details}

\paragraph{Latent prediction and SIGReg.}
For the implemented one-step objective, the causal predictor consumes $\mathbf{z}_{1:T}$ and is aligned with targets $\mathbf{z}_{2:T+1}$ according to Eq.~\ref{eq:pred_loss}. The target embeddings are not stop-gradient targets, and we do not use a target encoder or exponential-moving-average encoder. Instead, following LeJEPA and LWM \cite{balestriero2025lejepa, maes2026leworldmodel}, we apply Sketched Isotropic Gaussian Regularization (SIGReg) \cite{balestriero2025lejepa} to prevent collapse and stabilize training.

SIGReg is applied independently to the batch of latent representations at each time step and averaged over time. It matches random one-dimensional projections of the latent distribution to a standard Gaussian using the Epps--Pulley characteristic-function statistic. We use 1,024 random projections and 17 trapezoidal quadrature points over $[-5,5]$.

\paragraph{Temporal context regularization.}
We use a temporal adaptation of VICReg \cite{bardes2022vicreg} to encourage the residual context to encode static or slowly varying information complementary to the predictive latent. We first flatten the $Q$ context tokens:
\begin{equation}
\mathbf{v}_{bt}
=
\operatorname{vec}(\mathbf{u}_{bt})
\in\mathbb{R}^{d_uQ},
\qquad
\overline{\mathbf{v}}_b
=
\frac{1}{T}\sum_{t=1}^{T}\mathbf{v}_{bt}.
\end{equation}
The invariance term suppresses within-clip temporal variation:
\begin{equation}
\mathcal{L}_{\mathrm{inv}}
=
\frac{1}{BTd_uQ}
\sum_{b,t}
\left|
\mathbf{v}_{bt}
-
\overline{\mathbf{v}}_b
\right|_2^2.
\label{eq:u_invariance}
\end{equation}
To prevent a constant context code, the variance and covariance terms are computed across clip-level means:
\begin{align}
\mathcal{L}_{\mathrm{var}}
&=
\frac{1}{d_uQ}
\sum_j
\max\left(
0,,
\gamma
-
\sqrt{
\operatorname{Var}_b
[\overline{v}_{b,j}]
+\epsilon
}
\right),\\
\mathcal{L}_{\mathrm{cov}}
&=
\frac{1}{d_uQ}
\sum_{j\neq k}
C_{jk}^{,2},
\qquad
C
=
\operatorname{Cov}_b
(\overline{\mathbf{v}}_b).
\end{align}
The complete context regularizer is
\begin{equation}
\mathcal{L}_{u}
=
\alpha\mathcal{L}_{\mathrm{inv}}
+
\beta\mathcal{L}_{\mathrm{var}}
+
\delta\mathcal{L}_{\mathrm{cov}},
\label{eq:u_vicreg}
\end{equation}
with $(\alpha,\beta,\delta)=(25,25,1)$, $\gamma=1$, and $\epsilon=10^{-4}$. Thus, $\mathcal{L}_{\mathrm{inv}}$ encourages stability within a trajectory, whereas $\mathcal{L}_{\mathrm{var}}$ and $\mathcal{L}_{\mathrm{cov}}$ preserve variation across trajectories and prevent dimensional collapse.

\paragraph{Patch reconstruction.}
The reconstruction loss in Eq.~\ref{eq:rec_loss} is averaged over the batch, time, patch, and pixel dimensions. In the full model, neither $\mathbf{z}_t$ nor $\mathbf{u}_t$ is detached from this loss. Reconstruction therefore directly shapes the visual encoder and both latent streams, rather than serving only as a visualization objective. The loss weights are $(\lambda_z,\lambda_u,\lambda_{\mathrm{rec}})=(0.05,0.05,0.2)$.

\paragraph{LRC-JEPA training.}
Training clips contain three context observations and one prediction
target. Consecutive visual observations are separated by five
low-level environment steps, whose actions are concatenated and
z-score normalized. Frames are normalized using ImageNet statistics. The main
simulation models are trained for $10$ epochs with batch size $128$ and
learning rate $5\times10^{-5}$; Bridge-v2 is trained for $20$ epochs
with batch size $192$ and learning rate $5\times10^{-4}$. $90/10$ training--validation split is used for all datasets.

Optimization uses AdamW with weight decay $10^{-3}$, linear warm-up
followed by cosine annealing, bfloat16 mixed-precision training, and
gradient clipping at $1.0$. The decoder has a separate optimizer
parameter group, allowing its learning rate to be varied independently;
the main setting uses the same learning rate for the decoder and the
remaining network.

\paragraph{Latent Planning.}
After training, the context encoder and reconstruction decoder are not
needed for control. Given an observed history and a candidate action
sequence, the predictor is rolled out autoregressively in the
$\mathbf{z}$ space. For an image goal $\mathbf{i}_g$, the terminal cost
is
\begin{equation}
    \mathcal{C}
    \left(
        \mathbf{a}_{t:t+K-1}
    \right)
    =
    \left\|
        \widehat{\mathbf{z}}_{t+K}
        -
        \mathbf{z}_g
    \right\|_2^2,
    \qquad
    \mathbf{z}_g
    =
    g_z\!\left(
        \operatorname{Enc}_{\theta}(\mathbf{i}_g)_{\mathrm{cls}}
    \right).
    \label{eq:planning_cost}
\end{equation}
Candidate actions are optimized using the Cross-Entropy Method (CEM). In closed-loop control, only the first part of the optimized sequence is executed before obtaining a new observation and replanning. Because planning operates on a single $d_z$-dimensional state per frame rather than on hundreds of spatial tokens, it retains the computational efficiency of compact JEPA-style world models.

\paragraph{Planning protocol.}
For online control, we perform latent-space MPC
using CEM. The planning horizon contains five
action tokens, with each token representing five low-level actions.
CEM evaluates $300$ candidate sequences for $30$ optimization
iterations and retains the best $30$ candidates as elites. Candidate
sequences are scored by the squared distance between the terminal
predicted state $\mathbf{z}$ and the encoded goal state. The controller
executes five action tokens before replanning, with an overall budget
of $50$ low-level environment steps. Current and goal observations are
encoded once per MPC solve and their representations are reused across
CEM iterations, since they are independent of the candidate actions.
Bridge-v2 uses the same latent planning objective in an offline
start--goal evaluation. The demonstrated action sequence is ranked against 300 random candidate sequences using the predicted terminal latent distance to the goal, and recall at ranks 1, 5, and 10 and normalized expert rank are reported.

\subsection{Relation to Prior Designs}
LRC-JEPA is most directly related to LWM
\cite{maes2026leworldmodel}, from which it retains compact
action-conditioned prediction and SIGReg-based end-to-end training.
LWM, however, represents each frame using only one predictive
latent and entangles control-relevant and residual contextual information.
DINO-WM \cite{zhou2024dinowm} instead predicts frozen DINOv2 spatial
patch features, making its planning-time dynamics substantially more
expensive than global-latent prediction. Standard V-JEPA2-AC
\cite{assran2025vjepa2} similarly learns a large action-conditioned
predictor over frozen, pretrained spatial feature maps, whereas our
V-JEPA2-GAP-AC realization spatially pools these features before
learning the action-conditioned dynamics.

LRC-JEPA differs from these approaches by training a compact visual
state from scratch while retaining spatial and appearance information
in a separate, reconstruction-only context bottleneck. MC-JEPA
\cite{bardes2023mcjepa} jointly learns motion and content through shared
representation and optical-flow objectives, whereas LRC-JEPA assigns
the two streams asymmetric world-model roles: only $\mathbf{z}$ is
predicted and planned through, while $\mathbf{u}$ supplies temporally
regularized residual context. The reconstruction decoder is related to cross-attention masked decoders
\cite{fu2024crossmae}, but queries the compact
$[\mathbf{z}_t;\mathbf{u}^{\mathrm{dec}}_t]$ representation rather
than the complete set of visible patch tokens.

\section{Theoretical Analysis for LWM}
\label{sec:anal_lwm}

\subsection{Definitions and assumptions}
\label{app:defs}

\begin{definition}[Optimal-representation criteria, {\citealp{achille2018emergence}}]
\label{def:props}
Let $X$ be the input, $Y$ the task, $N$ a nuisance, and $Z$ a stochastic
representation of $X$ with coordinates $Z_1,\dots,Z_d$. Then $Z$ is
\emph{sufficient} for $Y$ if $I(Y;Z)=I(Y;X)$; \emph{minimal} if it minimizes
$I(Z;X)$ among sufficient representations; \emph{invariant} to $N$ if
$I(Z;N)=0$; and \emph{disentangled} if, among minimal sufficient
representations, it minimizes
\begin{equation}
  \TC(Z)\;:=\;\sum_{i=1}^{d} H(Z_i)-H(Z)
  \;=\;\KL\!\Big(q(\zvec)\,\Big\|\,\textstyle\prod_{i=1}^{d}q(z_i)\Big)\ \ge\ 0,
  \label{eq:tc-def}
\end{equation}
which is zero if $Z_1,\dots,Z_d$ are mutually independent.
\end{definition}

\begin{remark}[Tightness residual in Theorem~\ref{thm:criterion}(a)]
\label{rem:residual}
The nuisance attains the bound up to a residual
$\epsilon:=I(Z;Y\mid N)-I(X;Y)\in[0,H(Y\mid X)]$, which is zero whenever $Y$
is a deterministic function of $X$
\citep[Prop.~3.1]{achille2018emergence}.
\end{remark}

\begin{remark}[The task variable]
\label{rem:task}
LWM has no external label. We take the training task to be next-latent
prediction, $Y=\zvec_{t+1}$, so sufficiency is the fixed-point property that
$\zvec_t$ retains what is needed to predict $\zvec_{t+1}$ given $\avec_t$.
Squared error $\Lcal_{\mathrm{pred}}$ is a surrogate for this property, not a
proof of information-theoretic sufficiency. The minimality and $\TC$ are defined on the marginal of $Z$.
\end{remark}

\begin{assumption}[Noisy surrogate encoder]
\label{ass:surrogate}
If the encoder $\enc_\theta$ is deterministic, $I(Z;X)$ as written is infinite.
We analyze the standard noisy surrogate
$\tilde{\zvec}=\enc_\theta(\ovec)+\bm{\varepsilon}$,
$\bm{\varepsilon}\sim\Ncal(0,\sigma^2 I_d)$, at fixed $\sigma>0$, and recover
the deterministic map as $\sigma\to0^{+}$
\citep{alemi2017deep,achille2018information}. Then
$I(Z;X)=H(Z)-H(\varepsilon)$. Matching $q(\zvec)$ to $\Ncal(0,I_d)$ fixes the scalar mutual information $I(Z;X)$ at its Gaussian value, but leaves unspecified how this information budget is allocated across coordinates of $X$.
\end{assumption}

\begin{assumption}[Single-factor Gaussian controlled world]
\label{ass:1gauss}
Let $s\in\Rb^n$ be the unobserved state and $X=g(s)$ with $g$ injective on the
support of $s$. Proposition~\ref{prop:LeWorldModel}(b) invokes the setting in which
action-conditioned latent prediction plus an isotropic-Gaussian constraint
identifies $s$ up to $O(n)$~\citep{klindt2026lejepa,zhang2026controlled}:
$s\sim\Ncal(0,I_n)$; stationary linear-Gaussian dynamics
$s_{t+1}=As_t+B\avec_t+\xi_t$ with $s_{t+1}\sim\Ncal(0,I_n)$; jointly Gaussian
$(s_t,\avec_t)$; spectral separation of the predictable signal; and
non-degenerate conditional action variation given $s_t$. 
\end{assumption}

\begin{assumption}[Population optimality and realizability]
\label{ass:sigreg}
The deterministic predictive encoder has dimension $d_z=d_s$.
The encoder and predictor jointly attain a global minimum
of the population one-step LWM objective at which the ideal
population SIGReg loss vanishes.
The encoder class contains a map satisfying
$f(g(s))=Qs$ for some $Q\in O(d_s)$.
The predictor is continuous, and its class can realize the
conditional mean $\mathbb{E}[Z_{t+1}\mid Z_t,a_t]$
for every admissible encoder.
\end{assumption}

\subsection{The decomposition lemma}
\label{app:decomp}

\begin{lemma}[Exact information decomposition; {\citealp{chen2018isolating}},
cf.\ {\citealp{achille2018emergence,achille2018information}}]
\label{lem:decomp}
Let $k$ index a training set of size $K$ with $p(k)=1/K$, $q(\zvec\mid k)$ the
(surrogate) encoder, $q(\zvec)=\sum_k q(\zvec\mid k)p(k)$ the aggregate, and
$p(\zvec)=\prod_i p(z_i)$ a factorized prior. Then
\begin{equation}
\begin{aligned}
  \mathbb{E}_{p(k)}\!\big[\KL(q(\zvec\mid k)\,\|\,p(\zvec))\big]
  &= \underbrace{I_q(Z;k)}_{\text{\rm (i) rate / minimality}}
   + \underbrace{\TC(Z)}_{\text{\rm (ii) total correlation}} \\
  &\quad + \underbrace{\sum_{i=1}^{d}\KL\!\big(q(z_i)\,\|\,p(z_i)\big)}_{\text{\rm (iii) coordinate-wise matching}} ,
\end{aligned}
  \label{eq:decomp}
\end{equation}
an exact split into three non-negative terms. Term \textup{(i)} is the
index-code mutual information, the empirical counterpart of $I(Z;X)$;
\textup{(ii)} is \eqref{eq:tc-def}; \textup{(iii)} depends only on coordinate
marginals. The three are separately controllable: \textup{(iii)} can vanish
while \textup{(i)} takes any compatible value and while \textup{(ii)} remains
positive (correlated or copula-dependent coordinates with Gaussian margins).
\end{lemma}

\begin{proof}
Write $q(\zvec,k)=q(\zvec\mid k)p(k)$ and $q(\zvec)=\sum_k q(\zvec\mid k)p(k)$.
Insert $q(\zvec)$ and $\prod_i q(z_i)$ inside the logarithm:
{\footnotesize
\begin{align*}
\mathbb{E}_{p(k)}\!\big[\KL(q(\zvec\mid k)\,\|\,p(\zvec))\big]
&= \mathbb{E}_{q(\zvec,k)}\!\left[\log\frac{q(\zvec\mid k)}{p(\zvec)}\right] \\
&= \mathbb{E}_{q(\zvec,k)}\!\left[
      \log\frac{q(\zvec,k)}{q(\zvec)p(k)}
    + \log\frac{q(\zvec)}{\prod_i q(z_i)}
    + \log\frac{\prod_i q(z_i)}{\prod_i p(z_i)}\right] \\
&= \underbrace{\KL\!\big(q(\zvec,k)\,\|\,q(\zvec)p(k)\big)}_{{\rm (i)}}
 + \underbrace{\KL\!\Big(q(\zvec)\,\Big\|\,\textstyle\prod_i q(z_i)\Big)}_{{\rm (ii)}} \\
&\quad + \underbrace{\sum_i \KL\!\big(q(z_i)\,\|\,p(z_i)\big)}_{{\rm (iii)}}.
\end{align*}}
The factors telescope, so the identity is exact; each summand is a KL, hence
non-negative. Over sufficiently expressive stochastic encoder famillies, separate controllability follows because the terms are
functionals of $q(\zvec,k)$, of the dependence in $q(\zvec)$, and of the
coordinate marginals, respectively.
\end{proof}

\begin{remark}[What SIGReg targets]
\label{rem:sigreg-vs-kl}
LWM does \emph{not} optimize the left-hand side of \eqref{eq:decomp}. That KL
includes the rate and is the factorized-prior regularizer of
\citet{achille2018information}. SIGReg is a statistic of the aggregate
$q(\zvec)$: Epps--Pulley tests of random projections against $\Ncal(0,1)$
\citep{balestriero2025lejepa,maes2026leworldmodel}. Its population minimizer
$q(\zvec)=\Ncal(0,I_d)$ makes \textup{(ii)} and \textup{(iii)} vanish together
and leaves \textup{(i)} unconstrained. 
\end{remark}

\subsection{Proof of Proposition~\ref{prop:LeWorldModel}(Sufficiency)}
\label{app:prop-indep}
For LWM, SIGReg averages the Epps--Pulley normality statistic over random one-dimensional projections. The SIGReg objective is zero when almost every projection of $Z$ is standard normal; continuity of the characteristic function extends this to all directions, and the Cram'er--Wold theorem then implies $q(\zvec)=\Ncal(0,I_d)$. Hence the learned latent variables are mutually independent in any orthonormal basis. Under Assumptions~\ref{ass:1gauss} and~\ref{ass:sigreg},
the deterministic predictive representation satisfies
$Z_t=Qs_t$ almost surely for some orthogonal transformation $Q\in O(d_s)$.
Consequently, $s_t=Q^\top Z_t$, establishing state sufficiency.\hfill$\square$

\subsection{Proof of Proposition~\ref{prop:LeWorldModel}(Minimality)}
\label{app:prop-ident}
The objective $\Lcal_{\mathrm{LeWM}}$ targets unsupervised latent representation learning through $\Lcal_{\mathrm{pred}}$ and SIGReg. While $\Lcal_{\mathrm{pred}}$ encourages retention of transition-relevant information, SIGReg depends only on the aggregate $q(\zvec)$ and is therefore insensitive to conditional changes that alter $I(Z;X)$; its Gaussian target primarily prevents collapse. Hence, $\Lcal_{\mathrm{LeWM}}$ neither explicitly targets nor enforces minimality.
\hfill$\square$

\section{Theoretical Analysis for Proposed Method}
\label{sec:anal_lrc}
\subsection{Optimality of LRC-JEPA}
\label{app:lrc}

\begin{assumption}[Two-factor Gaussian controlled world]
\label{ass:2gauss}
Observations factor as $X_t=g(s_t,v)$ with $g$ injective on the joint
support. The dynamic state $s_t\in\Rb^{d_s}$ follows linear-Gaussian
controlled dynamics, with $s_t\sim\Ncal(0,I_{d_s})$, spectral separation
such that the predictable signal spans $s$, and action excitation. The
appearance factor $v\in\Rb^{d_v}$ is constant
within each trajectory, independent of $(s_{1:T},\avec_{1:T})$, and
non-degenerate across trajectories; $V$ is a nuisance for the prediction
task in the sense of Theorem~\ref{thm:criterion}, $I(Y;V)=0$.
\end{assumption}

\begin{assumption}[Global minimizer of $\Lcal_{LRC}$]
\label{ass:lrc}
Suppose the representation rate is minimum at convergence, enforced by architectural capacity controls $d_z=d_s, Qd_u=d_v$. At the global optimum of $\Lcal_{LRC}$, $\Lcal_u=0$ and the context embedding $u_t$ is constant within a trajectory while
retaining cross-trajectory variance (the temporal-VICReg population
minimizer).
\end{assumption}

With above anlaysis, we are ready to prove the our key theorem: 
\begin{proof}[Proof of Theorem~\ref{thm:lrc}]
We first show joint realizability of all loss terms, then use the context and reconstruction constraints to separate nuisance and state information, and finally invoke SIGReg and identifiability to establish invariance, minimality, coordinate disentanglement, and state recovery.

Consider $\zvec_t=Qs_t$ and $\uvec_t=\rho(v)$, where $Q\in O(n)$ and $\rho$ is injective and realizable by the context head. With predictor $\widehat{\zvec}_{t+1}=QAQ^\top\zvec_t+QB\avec_t$ and a decoder recovering $(s_t,v)$ from $(\zvec_t,\uvec_t)$ before applying $g$, the objective simultaneously attains the minimum prediction error for $\Lcal_{\mathrm{pred}}$, $\Lcal_{\mathrm{SIG}}=0$, the minimum of $\Lcal_u$, and $\Lcal_{\mathrm{rec}}=0$. Hence every global minimizer attains the minimum of each term individually.

\paragraph{Step 1: Invariance restricts $U_t$ to shared context.} In particular, $\Lcal_{\mathrm{SIG}}=0$ implies $q(\zvec)=\Ncal(0,I_n)$ and $\TC(Z)=0$ by Appendix~\ref{app:prop-indep}, \textit{proving~(iv) disentanglement in the Theorem~\ref{thm:lrc}}. Minimizing $\Lcal_u$ makes $\uvec_t$ a.s.\ constant along each trajectory. Conditional on almost every context $v$, the stationary state process is ergodic, $\omega(g(\cdot,v))$ is a.e.\ constant in $s$; thus $\uvec_t=\rho(v)$. 

\paragraph{Step 2: Compactness makes sufficiency.}
Since $U$ is a deterministic function of $V$,
\begin{equation}
    H(U)
    =
    I(U;V)
    =
    H(V)-H(V\mid U).
    \label{eq:u-rate}
\end{equation}

Perfect reconstruction gives
\begin{equation}
    H(V,S_t\mid U,Z_t)=0.
\end{equation}
Therefore,
\begin{align}
    H(V,S_t\mid U)
    &=
    I(V,S_t;Z_t\mid U)\\
    &\leq H(Z_t\mid U)\\
    &\leq H(Z_t).
\end{align}
Because $S_t$ is independent of $(V,U)$,
\begin{equation}
    H(V,S_t\mid U)
    =
    H(V\mid U)+H(S_t).
\end{equation}
Thus,
\begin{equation}
    H(Z_t)
    \geq
    H(V\mid U)+H(S_t).
    \label{eq:z-rate-lower-bound}
\end{equation}

Combining \eqref{eq:u-rate} and
\eqref{eq:z-rate-lower-bound} yields
\begin{align}
    \mathcal{R}
    &=
    H(U)+\sum_{t=1}^{T}H(Z_t)\\
    &\geq
    H(V)
    +
    \sum_{t=1}^{T}H(S_t)
    +
    (T-1)H(V\mid U).
    \label{eq:rate-lower-bound}
\end{align}

Since $T>1$, every
minimum-rate solution of \eqref{eq:rate-lower-bound} must satisfy
\begin{equation}
    H(V\mid U)=0.
    \label{u-sufficiency}
\end{equation}
Hence $U$ is sufficient for $V$.

Exact reconstruction then gives $D_\psi(\zvec_t,\rho(v))=X_t$ a.s., and injectivity of $g$ implies $s_t=m(\zvec_t,\rho(v))$ for some measurable $m$. Therefore the state must be recoverable from $\zvec_t$ once $v$ is fixed. Since $U_t$ is sufficient for $V$, $Z_t$ must be sufficient for $S_t$ conditioned on $U_t$, \textit{proving~(i) sufficiency in the Theorem~\ref{thm:lrc}}.

\paragraph{Step 3: Sufficiency leads to minimality and invariance} Combining this routing constraint with $s\perp v$, $s\sim\Ncal(0,I_n)$, $q(\zvec)=\Ncal(0,I_n)$, and Assumption~\ref{ass:surrogate} yields $I(Z;V)=0$, \textit{proving~(iii) invariance in the Theorem~\ref{thm:lrc}}. Consequently, $\zvec_t$ is a.e.\ a function of $s_t$ alone and $I(Z;X)=I(Z;S)$. Since the predictable signal spans $s$ by Assumption~\ref{ass:lrc}, every prediction-sufficient representation $Z'$ satisfies $I(Z';X)\ge I_\sigma(S;X)$, while $Z$ attains this lower bound; hence $Z$ is minimal, \textit{proving~(ii) minimality in the Theorem~\ref{thm:lrc}}.

In fact, $\zvec$ is an isotropic Gaussian representation of $s$ that minimizes next-latent MSE under the assumed linear-Gaussian controlled dynamics. By \cite{klindt2026lejepa} and \cite{zhang2026controlled}, $\zvec_t$ is equivalent to $\mathbf{s}_t$ up to an orthogonal transformation.
\end{proof}

\section{Additional Results}

We assess physical information accessible from LRC-JEPA's and LWM's predictive latent $\zvec$ using ridge regression and MLP probes. Probes are trained on frozen, converged encoder outputs and evaluated on held-out data. More details on probe fitting are provided in Appendix \ref{sec:imp_details}
. Table \ref{tab:pusht_probe_comparison} and \ref{tab:tworoom_probe_comparison} report probing results on the two simulated environments, Push-T and TwoRoom, where the ground truth quantities are provided by the simulator. Table \ref{tab:cube_probe_z_plus_u} shows probing results with the concatenation of the two embeddings ($\zvec$ and $\uvec$) on Cube environment. Table \ref{tab:bridge_probe} and \ref{tab:bridge_functional_separation_z_plus_u} further analyze LRC-JEPA's predictive latents and context embeddings on the real-world Bridge-v2 dataset, where $Q=2$ embeddings are used for context representation.

\begin{table}[H]
\caption{Physical quantity probes on the simulated Push-T environment comparing LRC-JEPA and LWM. Each entry reports the mean-squared error (MSE; lower is better) and the Pearson correlation coefficient ($r$; higher is better). Both probe heads are fitted using the latent representation $z$.}
\label{tab:pusht_probe_comparison}
\small
\centering
\setlength{\tabcolsep}{6pt}
\begin{tabular}{llcccc}
\toprule
\multirow{2}{*}{\textbf{Quantity}}
& \multirow{2}{*}{\textbf{Model}}
& \multicolumn{2}{c}{\textbf{$z$ Linear}}
& \multicolumn{2}{c}{\textbf{$z$ MLP}} \\
\cmidrule(lr){3-4}\cmidrule(lr){5-6}
& & \textbf{MSE $\downarrow$} & \textbf{$r \uparrow$}
  & \textbf{MSE $\downarrow$} & \textbf{$r \uparrow$} \\
\midrule

\multirow{2}{*}{Agent position}
& LRC-JEPA & 0.066 & 0.967 & 0.022 & 0.989 \\
& LWM      & 0.044 & 0.978 & 0.018 & 0.991 \\
\addlinespace

\multirow{2}{*}{Block position}
& LRC-JEPA & 0.021 & 0.989 & 0.005 & 0.997 \\
& LWM      & 0.021 & 0.990 & 0.006 & 0.997 \\
\addlinespace

\multirow{2}{*}{Block angle}
& LRC-JEPA & 0.175 & 0.905 & 0.049 & 0.974 \\
& LWM      & 0.184 & 0.900 & 0.055 & 0.971 \\
\bottomrule
\end{tabular}
\end{table}

\begin{table}[H]
\caption{Physical quantity probes on the simulated TwoRoom environment comparing LRC-JEPA and LWM. Each entry reports the mean-squared error (MSE; lower is better) and the Pearson correlation coefficient ($r$; higher is better). Both probe heads are fitted using the latent representation $z$.}
\label{tab:tworoom_probe_comparison}
\small
\centering
\setlength{\tabcolsep}{6pt}
\begin{tabular}{llcccc}
\toprule
\multirow{2}{*}{\textbf{Quantity}}
& \multirow{2}{*}{\textbf{Model}}
& \multicolumn{2}{c}{\textbf{$z$ Linear}}
& \multicolumn{2}{c}{\textbf{$z$ MLP}} \\
\cmidrule(lr){3-4}\cmidrule(lr){5-6}
& & \textbf{MSE $\downarrow$} & \textbf{$r \uparrow$}
  & \textbf{MSE $\downarrow$} & \textbf{$r \uparrow$} \\
\midrule

\multirow{2}{*}{Agent position}
& LRC-JEPA & 0.002 & 0.999 & 0.000 & 1.000 \\
& LWM      & 0.008 & 0.996 & 0.001 & 1.000 \\
\bottomrule
\end{tabular}
\end{table}

\begin{table}[H]
\caption{Physical quantity probes using the concatenated $\zvec\oplus \uvec$ representation of LRC-JEPA on the Cube environment. Each entry reports MSE / Pearson correlation $r$ (lower / higher is better). Here, $\oplus$ denotes concatenation.}
\label{tab:cube_probe_z_plus_u}
\centering
\small
\setlength{\tabcolsep}{4.5pt}
\renewcommand{\arraystretch}{1.05}
\begin{tabular}{@{}lcc@{}}
\toprule
\textbf{Quantity}
& \textbf{Linear}
& \textbf{MLP} \\
\midrule

Joint position
& 0.272 / 0.731
& 0.306 / 0.711 \\

Joint velocity
& 0.746 / 0.397
& 1.002 / 0.088 \\

End-effector position
& 0.009 / 0.996
& 0.019 / 0.991 \\

End-effector yaw
& 0.891 / 0.323
& 0.923 / 0.299 \\

Gripper
& 0.058 / 0.970
& 0.060 / 0.969 \\

Block position
& 0.004 / 0.998
& 0.011 / 0.995 \\

Block quaternion
& 0.673 / 0.160
& 0.711 / 0.062 \\

Block yaw
& 0.989 / 0.173
& 1.031 / 0.150 \\
\bottomrule
\end{tabular}
\end{table}

\begin{table}[H]
\centering
\small
\setlength{\tabcolsep}{4pt}
\caption{Object-centric physical quantity probes using $\zvec$ and mean-pooled ($Q=2$) $\uvec$ on Bridge-v2. Each entry reports MSE / Pearson correlation $r$. One valid moving, static, and background object is selected per frame. Bold marks the better representation for each probe type and metric.}
\label{tab:bridge_probe}
\begin{tabular}{@{}lcccc@{}}
\toprule
& \multicolumn{2}{c}{$z$} & \multicolumn{2}{c}{$u$} \\
\cmidrule(lr){2-3}\cmidrule(lr){4-5}
Probe target & Linear & MLP & Linear & MLP \\
\midrule
Gripper centroid
& \textbf{0.724 / 0.690} & \textbf{0.604 / 0.735}
& 1.193 / 0.321 & 0.858 / 0.402 \\
Moving-object centroid
& \textbf{1.001 / 0.484} & \textbf{0.879 / 0.473}
& 1.528 / 0.225 & 0.995 / 0.323 \\
Moving-object RGB
& 2.440 / 0.131 & \textbf{1.141 / 0.269}
& \textbf{1.914 / 0.208} & 1.179 / 0.265 \\
Static-object centroid
& 1.841 / \textbf{0.252} & 1.277 / 0.259
& \textbf{1.508} / 0.246 & \textbf{1.239 / 0.267} \\
Static-object RGB
& 2.324 / 0.036 & 1.208 / 0.326
& \textbf{2.035 / 0.137} & \textbf{0.999 / 0.413} \\
Background centroid
& 1.241 / \textbf{0.425} & \textbf{0.554 / 0.609}
& \textbf{1.211} / 0.337 & 0.684 / 0.489 \\
\bottomrule
\end{tabular}
\end{table}

\begin{table}[H]
\centering
\small
\setlength{\tabcolsep}{4pt}
\caption{
Object-centric physical quantity probes using the concatenated $\zvec\oplus \uvec$ representation on Bridge-v2. Each entry reports MSE / Pearson correlation $r$ (lower / higher is better) for linear and MLP probes. Here, $\uvec$ is mean-pooled ($Q=2$) before concatenation. One valid moving, static, and background object is selected per frame.
}
\label{tab:bridge_functional_separation_z_plus_u}
\begin{tabular}{@{}lcc@{}}
\toprule
Probe target
& Linear
& MLP \\
\midrule

Gripper centroid
& 10.829 / 0.299
& 0.703 / 0.690 \\

Moving-object centroid
& 15.260 / 0.144
& 0.888 / 0.460 \\

Moving-object RGB
& 13.374 / 0.165
& 1.277 / 0.258 \\

Static-object centroid
& 13.174 / $-0.018$
& 1.219 / 0.289 \\

Static-object RGB
& 11.721 / 0.204
& 1.200 / 0.347 \\

Background centroid
& 4.750 / 0.156
& 0.598 / 0.571 \\

\bottomrule
\end{tabular}
\end{table}

\subsection{Reconstruction Visualization}
\begin{figure}[H]
\centering
\includegraphics[width=0.8\linewidth]{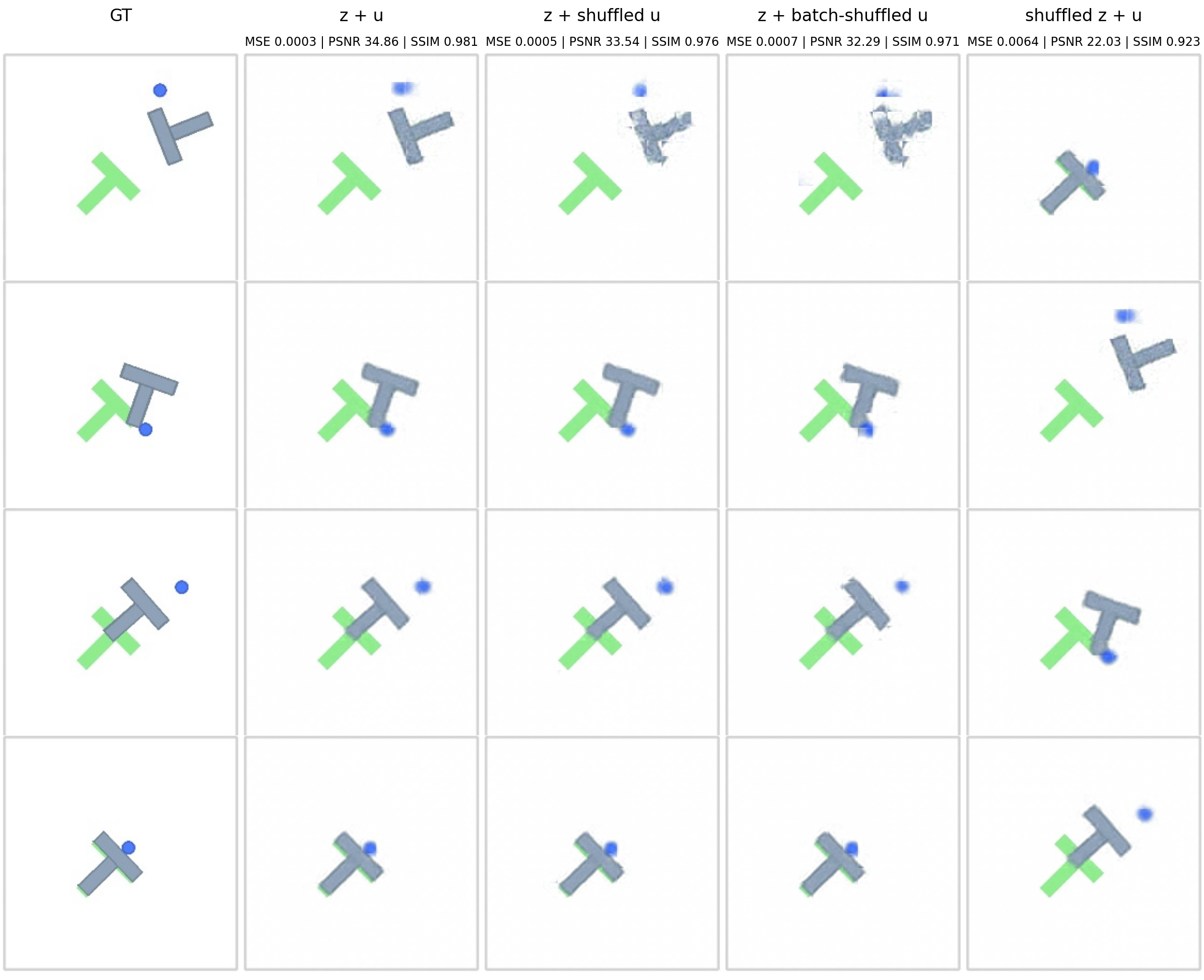}
\caption{Push-T reconstructions. ``Shuffled'' permutes frames within a trajectory; ``batch-shuffled'' exchanges context across trajectories.}
\label{fig:pusht}
\end{figure}

\begin{figure}[H]
\centering
\includegraphics[width=0.8\linewidth]{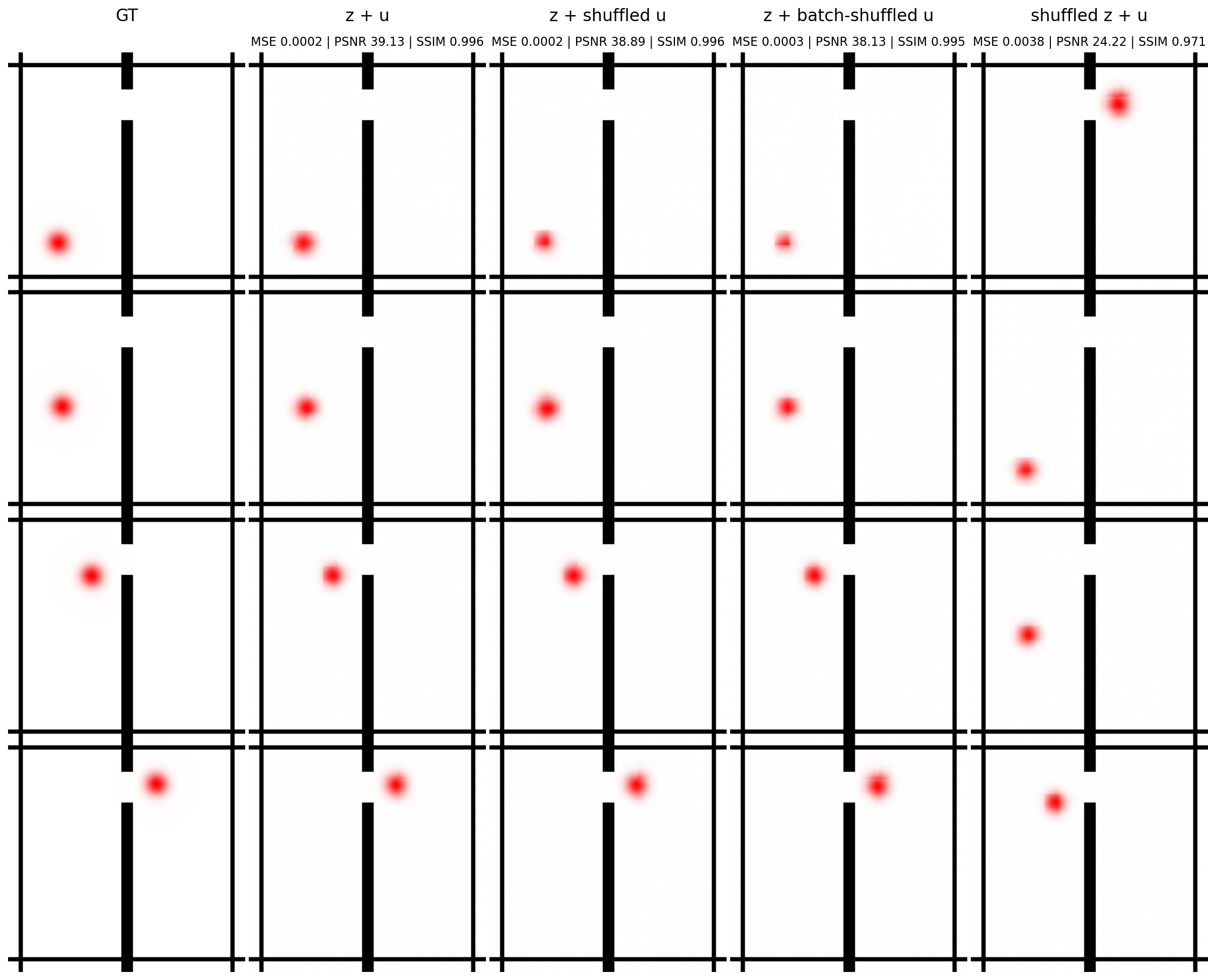}
\caption{TwoRoom reconstructions under the same within-trajectory and across-trajectory interventions.}
\label{fig:tworoom}
\end{figure}

\end{document}